\documentclass[11pt]{article}

\usepackage[utf8]{inputenc} 
\usepackage[T1]{fontenc}    
\usepackage{hyperref}       
\usepackage{url,authblk}            
\usepackage{booktabs}       
\usepackage{amsfonts}       
\usepackage{nicefrac}       
\usepackage{microtype}      
\usepackage{xcolor}         

\usepackage[round]{natbib}
\usepackage{graphicx}
\usepackage{amsmath,amssymb,mathtools,amsthm}
\usepackage{mathrsfs}
\usepackage{dirtytalk}
\usepackage{tikz}
\usepackage{placeins}
\usepackage{csvsimple}
\usepackage{siunitx}
\usepackage{algorithm}
\usepackage{algorithmic}
\newcommand{\Done}{D}
\newcommand{\Dtwo}{D^{\prime}}
\newcommand{\Vone}{V}
\newcommand{\Vtwo}{V^{\prime}}
\newcommand{\VDone}{V_{\Done}}
\newcommand{\VDtwo}{V_{\Dtwo}}
\newcommand{\pione}{\pi}
\newcommand{\pitwo}{\pi^{\prime}}
\newcommand{\piDone}{\pi_{\Done}}
\newcommand{\piDtwo}{\pi_{\Dtwo}}
\newcommand{\rhoone}{\rho}
\newcommand{\rhotwo}{\rho^{\prime}}
\newcommand{\Vsenbound}{\Delta_{\operatorname{pot}}}
\newcommand{\Vgradsenbound}{\Delta_{\operatorname{gradpot}}}
\newcommand{\Vsenboundalt}{\Delta_{\operatorname{pot},\operatorname{osc}}}
\newcommand{\tarloglikbound}{S_{\operatorname{tar}}}
\newcommand{\tarloglikbounddata}{S_{\operatorname{tar}}^{\Done,\Dtwo}}
\newcommand{\SHK}{\textrm{SHK}}
\newcommand{\lambdapi}{\lambda_{\mathrm{Gibbs}}}
\newcommand{\lambdarho}{\lambda_{\SHK}}

\newcommand{\KL}{\operatorname{KL}}
\newcommand{\KLf}[2]{\operatorname{KL}(#1 \mid\mid #2)}
\newcommand{\TV}[2]{\operatorname{TV}(#1 \mid\mid #2)}
\newcommand{\Hock}[2]{\operatorname{E}_{\varepsilon}(#1 \mid\mid #2)}

\newcommand{\Renyi}[2]{\operatorname{D}_{\alpha}(#1 \mid\mid #2)}

\newcommand{\R}{\mathbb{R}}
\newcommand{\E}{\mathbb{E}}

\newcommand{\Rd}{\mathbb{R}^{d}}

\newcommand{\Prob}{\mathcal{P}(\Theta)}
\newcommand{\Probtheta}{\mathcal{P}_{2}(\Theta)}

\newcounter{lemmano}
\newcounter{theoremno}
\newcounter{propositionno}

\newcounter{definitionno}
\newcounter{remarkno}

\newtheorem{theorem}[theoremno]{Theorem}
\newtheorem{lemma}[lemmano]{Lemma}
\newtheorem{proposition}[propositionno]{Proposition}

\newtheorem{definition}[definitionno]{Definition}
\newtheorem{remark}[remarkno]{Remark}

\newenvironment{theorem*}{{\bf Lemma:}}

\newtheorem{asu}{Assumption}

\newtheorem{asuprime}{Assumption}

\title{On the Stability of Spherical Hellinger-Kantorovich Flows and Their Implications for Differential Privacy}

\author{Aratrika Mustafi*}
\author{Soumya Mukherjee}

\affil{Department of Statistics, Pennsylvania State University}

\begin{document}

\maketitle

\begin{abstract}
  Gradient-flow sampling interprets a Gibbs distribution as the minimizer of an energy functional over probability measures and generates dynamics converging to this target. Under spherical Hellinger-Kantorovich (SHK) geometry, the flow couples transport and reaction and coincides with birth-death Langevin dynamics. In this work, we develop a perturbation theory for SHK gradient flows. For two potentials $\Vone$ and $\Vtwo$ we compare the associated flows from a common initialization and quantify how potential discrepancies propagate over time. A uniform perturbation bound yields dimension-free, pointwise control of the log-likelihood ratio and Rényi divergence, while additional structure allows us to derive bounds for the KL divergence as well. We apply these results to approximate sampling for the exponential mechanism in differential privacy. The likelihood-ratio control provides explicit time-dependent Pure-DP guarantees for SHK-based samplers, while the KL bound yields Approximate-DP certificates via hockey-stick divergence. We also derive a utility bound separating intrinsic exponential-mechanism suboptimality from finite-time sampling error.
\end{abstract}

\section{Introduction}\label{sec:intro}
We consider the problem of sampling from an unnormalized Boltzmann/ Gibbs density, $ \pi(\theta)\ \propto\ \exp\big(-\Vone(\theta)\big), \theta\in\Theta\subset\Rd$ , where the normalization constant is unknown (and/or intractable) and only the potential function $\Vone$ (and typically its derivatives) can be evaluated. This problem arises across various domains in Bayesian inference, statistical physics, and modern machine learning. A common variational perspective on sampling is to characterize the target distribution $\pi$ as the unique minimizer of a functional (typically a divergence functional) over the space of probability measures. From this viewpoint, sampling can be formulated as evolving an initial distribution $\rho_0$ toward $\pi$ via the gradient flow of this functional under a suitable geometric structure on the space of probability measures.

In this paper, we focus on a gradient flow based sampling methodology built from the spherical Hellinger Kantorovich (SHK),  also known as the Wasserstein Fisher Rao (WFR), geometry on the space of probability measures \citep{kondratyev2019spherical,liero2018optimal, chizat2015interpolating}. When the variational objective is the exclusive KL divergence $\rho \mapsto \KL(\rho\|\pi) $, the SHK gradient flow generates a time-indexed family of marginals $\{\rho_t\}_{t\geq 0}$ (initialized at $\rho_0 \in \Probtheta $ ) that evolves according to the continuity reaction equation \eqref{eq: SHK dynamics Fokker Planck Equation}. This evolution is equivalent to the birth-death Langevin dynamics introduced in \cite{lu2019accelerating} . The SHK geometry should be contrasted with the classical Wasserstein (aka Kantorovich) geometry - In the latter case, the gradient flow is purely transport based and corresponds to the standard continuity equation (Fokker Planck dynamics) \eqref{eq: overdamped Langevin Fokker Planck Equation}, whereas the SHK gradient flow augments transport with a reaction component giving rise to a continuity reaction equation which leads to provable acceleration relative to the transport only dynamics in challenging energy landscapes [See Section \ref{sec:Wassersein gradient flows and spherical Hellinger Kantorovich (SHK) gradient flows} for more details].

The primary theoretical objective of this work is to quantify the stability of the SHK geometry based sampler under perturbations of the potential. Specifically, given two potentials $\Vone$ and $\Vtwo$ (and hence two target distributions $\pione \propto \exp(-\Vone), \pitwo \propto \exp(-\Vtwo)$), we study the corresponding SHK flows starting from the same initialization $\rho_0$, and establish general bounds that control how a perturbation from $\Vone$ to $\Vtwo$ propagates to the time-marginal densities $\rho_t$ and $\rhotwo_t$. In particular, we bound, for every $t\geq 0$, different divergences between $\rho_t$ and $\rhotwo_t$, thereby providing a rigorous perturbation theory for SHK gradient flow iterates.

Next, as an application of this stability theory, we study differential privacy (DP) guarantees for approximate sampling procedures. In DP, a widely used randomized mechanism is the Exponential mechanism, which samples an output $\theta $ from a dataset dependent density,
\begin{equation}\label{eq:exponential mechanism density}
\piDone(\theta)\ \propto\ \exp\!\big(-\varepsilon L_D(\theta)\big)
\end {equation}
where $\varepsilon >0 $ is the privacy parameter that controls the amount of noise injected and $L_{D}(\theta)$ is a loss function parametrized by $\theta$. The loss function $L_D$ measures how well the parameter $\theta$ fits the dataset $D$, and is typically minimized at estimators that best explain or approximate the data. This is again a Gibbs distribution with potential $V_D \coloneqq \varepsilon L_D$. For two neighboring datasets $ \Done, \Dtwo$, privacy hinges on how the law of the released $\theta$ changes as $V_\Done$ changes to $V_\Dtwo$. 

These considerations motivate samplers that (i) operate without access to normalization constants, (ii) admit precise stability guarantees w.r.t perturbations of the potential and (iii) allow these stability guarantees to be translated into DP bounds. The SHK geometry based gradient flow is particularly appealing in this context - it combines transport with a reaction mechanism that reweights probability mass, thereby providing additional flexibility in adapting to changes in the potential. The SHK gradient flow generates a time-indexed family of marginal densities $\{\rho_{t}^{D}
\}_{\{t \geq 0\}}$ associated with the dataset-dependent potential $V_D$. By leveraging our perturbation bounds for SHK flows, we quantify how $\rho_{t}^D$ changes when the dataset $D$ is replaced by a neighboring dataset. These stability estimates can then be translated directly into explicit 
differential privacy guarantees for the time - $t$ distribution generated by the SHK gradient flow when used to approximate $\piDone$.
 \paragraph{Related works}
The intersection of differential privacy with stochastic optimization and diffusion based sampling is extensive. Foundational treatments of DP and mechanisms based on Gaussian noise, together with practical deployments such as DP-SGD, motivate analyzing privacy loss along stochastic dynamics \citep{dwork2014algorithmic, abadi2016deep}. For Langevin diffusion and noisy gradient descent, \citep{chourasia2021differential, altschuler2022privacy, ganesh2020faster, ganesh2022differentially} provide time-dependent privacy guarantees (including Rényi-divergence based analyses) and show how privacy risk can depend delicately on runtime, discretization, and noise structure. A complementary line of work relates privacy loss to information-theoretic quantities such as relative entropy and uses such bounds to control membership-inference and other information-extraction risks \citep{melis2019exploiting,mahloujifar2022optimal,wibisono2017information}. On the analytical side, stability of Fokker-Planck equations and distances between diffusions (and their invariant measures) have been studied through heatkernel bounds, ergodicity theory, and perturbation methods
\citep{mattingly2002ergodicity,pavliotis2014stochastic, sanz2016gaussian, bogachev2014kantorovich, bogachev2016distances}, often combined with functional inequalities and concentration tools
\citep{Bakry_2014,ledoux_1999,ledoux_2001,Malrieu_2001,schlichting2019poincare}. 

Closest to our DP-motivated stability viewpoint is the recent work of \citet{borovykh2023privacy}, which derives relative entropy stability bounds for Langevin-type SDEs with possibly anisotropic diffusion and translates them into $(\varepsilon, \delta)$-DP and membership inference guarantees; it also discusses privacy-accuracy tradeoffs for anisotropic or matrix-valued noise and related mechanisms.
Our work complements this literature by moving beyond transport-only Langevin flows; we analyze  the birth-death Langevin dynamics and establish perturbation-based stability bounds tailored to this geometry, which we then specialize to obtain differential privacy guarantees for exponential-mechanism samplers.

\section{Notation}
Let us denote the collection of probability measures on the domain $\theta \subset \Rd$ as $\Prob$ and let $\Probtheta \subset \Prob$ be the subset of probability distributions with finite second moments. Throughout this paper, we use the exclusive version of the KL divergence to compare an approximating distribution $\rho$ to a target distribution $\pi$ i.e. $\KLf{\rho}{\pi} \coloneqq \int_{\Theta} \log \left(\frac{\rho}{\pi}\right) d \rho$. For $\alpha \in(0,1) \cup(1, \infty)$, the Rényi divergence of order $\alpha$ is defined as $\Renyi{\rho}{\pi}=\frac{1}{\alpha-1} \log \int \left(\frac{\rho}{\pi}\right)^{\alpha - 1} \mathrm{d} \rho$. To bridge our perturbation results to differential privacy guarantees, we also introduce the \say{Hockey-stick} divergence \citep{sason2016hockeystickdivergence,minasyan2018hockey}. Let $\rho$ and $\pi$ be two probability measures on a measurable space $\mathcal\{\Theta, \mathcal{ F}\}$.
The Hockey-stick divergence is a parametrized extension of the Total Variation divergence and is defined as $\Hock{\rho}{\pi}= \sup_{S\in \mathcal{F}}\left[\rho(S)-\exp(\varepsilon) \pi(S)\right]$.

We denote the (spatial) gradient operator by $\nabla$, the (spatial) Hessian operator by $\nabla^2$, the divergence operator by $\nabla\cdot$ and the Laplacian operator by $\Delta$, where $\Delta f = \nabla\cdot (\nabla f)$. We denote the collection of real-valued functions defined on the domain $[0,T] \times \Theta$ which are $k_1$-times continuously differentiable in the first argument (for every fixed value of the second argument) and $k_2$-times continuously differentiable in the second argument (for every fixed value of the first argument) as $C^{k_1,k_2}([0,T] \times \Theta)$. On any given domain $\Omega$, let $C^k(\Omega)$ denote the collection of $k$-times continuously differentiable real-valued functions.

\section{ Wasserstein gradient flows and spherical Hellinger Kantorovich (SHK) gradient flows}  \label{sec:Wassersein gradient flows and spherical Hellinger Kantorovich (SHK) gradient flows} A canonical method for sampling from a density $\pi(\theta) \propto \exp (-V (\theta)), \theta \in \Theta \subset \mathbb{R}^{d}$, with unknown normalization constant is the overdamped Langevin diffusion 
\begin{equation}\label{eq: overdamped Langevin diffusion}
    d \theta_t=-\nabla V\left(\theta_t\right) d t+\sqrt{2} d W_t
\end{equation}
which is designed such that $\pi$ is an invariant distribution of this process and $W_t$ is a $d$-dimensional Brownian motion. The unadjusted Langevin algorithm (ULA) is obtained by an explicit time discretization of this diffusion and is used in large scale sampling because it only requires knowledge of $ \nabla V$.
Consequently, if $\theta_t$ has a marginal density $\rho_t$  then $\rho_t$ satisfies the Fokker-Planck equation (FPE)
\begin{equation}\label{eq: overdamped Langevin Fokker Planck Equation}
    \partial_t \rho_t=\nabla \cdot\left(\nabla \rho_t+\rho_t \nabla V\right).
\end{equation}
Since $ \nabla \cdot\left(\nabla \rho_t (\theta) +\rho_t (\theta) \nabla V (\theta)\right) =\nabla \cdot\left(\rho_t(\theta) \nabla \log \frac{\rho_t(\theta)}{\pi(\theta)}\right)$, the above PDE has a variational interpretation originating from the works of \citet{jordan1998variational}. Equation \eqref{eq: overdamped Langevin Fokker Planck Equation} can be viewed as the gradient flow (direction of steepest descent) of the $\KL$ divergence $ \rho \mapsto \KL(\rho \| \pi)$ on the space of probability measures endowed with the Wasserstein-2 $(W_2)$ geometry.

While the Wasserstein geometry provides a principled transport based descent mechanism, it is intrinsically mass conserving i.e. the evolution can only reallocate probability mass through advection/diffusion but cannot reweight them. In energy landscapes with well separated modes, this can lead to slow convergence of the $\KL(\rho_t \| \pi)$. A natural way to complement this conservative transport is to add a reaction term, which leads to a continuity-reaction of the form,
\begin{equation}\label{eq: SHK dynamics Fokker Planck Equation}
    \partial_t \rho_t=\nabla \cdot\left(\nabla \rho_t+\rho_t \nabla V\right)-\alpha_t \rho_t,
\end{equation}
where $\alpha_t=\log \rho_t-\log \pi-\int_{\Theta}\left(\log \rho_t-\log \pi\right) d\rho_t = \log \rho_t-\log \pi- \E_{\rho_t} \left(\log \rho_t-\log \pi\right)$ (See \citet[Theorem 3.1]{lu2019accelerating}).
Equation \eqref{eq: SHK dynamics Fokker Planck Equation} arises from the evolution of the $\KL$ functional on the space of probability measures endowed with the spherical Hellinger Kantorovich (aka Wasserstein Fisher Rao) geometry. The centering $\E_{\rho_t} \left(\log \rho_t-\log \pi\right)$ is introduced to ensure that the reaction component satisfies $\int \alpha_t \rho_t =0 $, thereby preserving the total mass and  guaranteeing $\rho_t$ remains a probability measure.

\section{Main Perturbation based Results}\label{sec: Main Perturbation based Results}

In this Section, we attempt to answer the central analytical question of this paper:
\emph{ How do SHK gradient flows respond to perturbations of the associated potential functions?} Our objective is to quantify, in a nonasymptotic manner, how a perturbation in the potential propagates through the transport-reaction dynamics and affects the evolving marginals.

Throughout this Section, we impose the following setting. Let $\Theta \subset \Rd$ be a compact domain without boundary. Consider two smooth ($C^2$) potentials
$
    V: \Theta \rightarrow \R, \quad \Vtwo: \Theta \rightarrow \R.
$
Let us define the corresponding Gibbs target distributions
\begin{equation*}
\pi(\theta)=\frac{\exp(-V(\theta))}{Z}, \quad Z\coloneqq \int_{\Theta} \exp(-V(\theta)) d \theta
\end{equation*}
\begin{equation*}    
\pitwo(\theta)=\frac{\exp(-\Vtwo(\theta))}{Z^{\prime}}, \quad Z^{\prime}\coloneqq\int_{\Theta} \exp(-\Vtwo(\theta)) d \theta .
\end{equation*}
Fix any $T>0$. Assume $\left\{\rho_t\right\}_{0 \leq t \leq T}$ and $\left\{\rhotwo_t\right\}_{0 \leq t \leq T}$ are strictly positive classical solutions
\begin{equation*}
\rho_t, \rhotwo_t \in C^{1,2}([0, T] \times \Theta), \quad \rho_t(\theta)>0, \rhotwo_t(\theta)>0 \quad \forall \theta \in \Theta
\end{equation*}
to the transport-reaction (birth-death accelerated Langevin) equations
\begin{equation}\label{eq: SHK FPE for potential V}
\partial_t \rho_t=\nabla \cdot(\nabla \rho_t+\rho_t \nabla V)-\alpha_t \rho_t,
\end{equation}
and
\begin{equation}\label{eq: SHK FPE for potential Vprime}
\partial_t \rhotwo_t=\nabla \cdot\left(\nabla \rhotwo_t+\rhotwo_t \nabla \Vtwo\right)-\alpha^{\prime}_t \rhotwo_t,
\end{equation}
respectively, where the (centered) birth-death rates are
\begin{equation*}
\alpha_t(\theta)=\log \rho_t(\theta)-\log \pi(\theta)-\int_{\Theta}(\log \rho_t-\log \pi) d\rho_t
\end{equation*}
and
\begin{equation*}
\alpha^{\prime}_t(\theta)=\log \rhotwo_t(\theta)-\log \pitwo(\theta)-\int_{\Theta}\left(\log \rhotwo_t-\log \pitwo\right) d\rhotwo_t .
\end{equation*}
By construction, $\pi$ and $\pitwo$ are invariant distributions of \eqref{eq: SHK FPE for potential V} and \eqref{eq: SHK FPE for potential Vprime}, respectively.
Further, assume that the density used as the initial iterate for both flows are the same i.e. $\rho_0=\rhotwo_0$.
Let us define 
$
H(t)\coloneqq \KLf{\rho_t}{\rhotwo_t}=\int_{\Theta} \log \left(\frac{\rho_t}{\rhotwo_t}\right) d \rho_t .
$

There are several natural ways to measure stability between the perturbed flows, such as pointwise control of the likelihood ratio $\frac{\rho_t}{ \rho_t^{\prime}}$, integral divergences such as KL or R\'enyi divergence, or weaker metrics such as total variation. We begin with the strongest possible control: an $L^{\infty}$-type bound on the log-likelihood ratio.

The complete version of the assumptions required to derive our results from here on are stated in the Appendix in Section \ref{sec: Assumptions}.

\begin{theorem}\label{log likelihood and Renyi divergence bound only under potential perturbation bound}
    Let $\pione \propto \exp(-\Vone)$ and $\pitwo \propto \exp(-\Vtwo)$ be Gibbs/Boltzmann corresponding to the potentials $\Vone$ and $\Vtwo$. Fix initial densities $\rho_0=\rhotwo_0$ on $\Theta$ which are strictly positive, and smooth enough so that the SHK Fokker Planck Equations \ref{eq: SHK FPE for potential V} and \eqref{eq: SHK FPE for potential Vprime} have positive classical solutions $\rho_t,$ and $\rhotwo_t$, respectively, for all $t \geq 0$.
    Define
    $$
    s_0(\theta)\coloneqq\log \frac{\rho_0(\theta)}{\pione(\theta)}, \quad s_0^{\prime}(\theta)\coloneqq\log \frac{\rho_0(\theta)}{\pitwo(\theta)}
    $$
    and
    $$
    R_0\coloneqq\max _{\theta \in \Theta}s_0(\theta) - \min _{\theta \in \Theta}s_0(\theta), \quad R_0^{\prime}\coloneqq\max _{\theta \in \Theta}s_0^{\prime}(\theta) - \min _{\theta \in \Theta}s_0^{\prime}(\theta).
    $$
    If Assumption \ref{Ass: Bound on size of perturbation of potentials} holds, i.e. $\left\|\Vone-\Vtwo\right\|_{\infty} \leq \Vsenbound$, then for every $t \geq 0$ and every $\theta \in \Theta$,
    $$
    \left|\log \frac{\rhoone_t(\theta)}{\rhotwo_t(\theta)}\right| \leq 2 \Vsenbound+\exp(-t)(R_0+R_0^{\prime})
    $$
    and for every $\alpha >1$, the $\alpha$-R\'enyi divergence satisfies
    $$
    \Renyi{\rho_t}{\rhotwo_t} \leq 2\Vsenbound + \exp(-t)(R_0+R_0^{\prime}).
    $$
    If instead Assumption \ref{Ass: Alternative Bound on size of perturbation of potentials} holds i.e. $\sup _{\theta \in \Theta}\left(\Vone(\theta)-\Vtwo(\theta)\right) -  \inf _{\theta \in \Theta}\left(\Vone(\theta)-\Vtwo(\theta)\right) \leq \Vsenboundalt$, then we have that, for every $t \geq 0$ and every $\theta \in \Theta$,
    $$
    \left|\log \frac{\rhoone_t(\theta)}{\rhotwo_t(\theta)}\right| \leq \Vsenboundalt+\exp(-t)(R_0+R_0^{\prime})
    $$
    and for every $\alpha >1$, the $\alpha$-R\'enyi divergence satisfies
    $$
    \Renyi{\rho_t}{\rhotwo_t} \leq \Vsenboundalt + \exp(-t)(R_0+R_0^{\prime}).
    $$
    
\end{theorem}
The proof is available in Section \ref{Proof: log likelihood and Renyi divergence bound only under potential perturbation bound} of the Appendix.
\begin{remark}
   Theorem\ref{log likelihood and Renyi divergence bound only under potential perturbation bound} gives a pointwise stability bound for the SHK gradient flow under a uniform perturbation of the potential (Assumption \ref{Ass: Bound on size of perturbation of potentials} or Assumption \ref{Ass: Alternative Bound on size of perturbation of potentials}). The estimate cleanly decomposes into: (i) a model/target mismatch term of order $\Vsenbound$/$\Vsenboundalt$, and (ii) an initialization term that decays like $e^{-t}$ through the oscillations $R_0,R_0'$ of the initial log-densities relative to the respective Gibbs targets.
This makes the role of the reaction component explicit: it exponentially damps the spatial oscillation of $s_t=\log(\rho_t/\pi)$ (and similarly for $s_{t}^{'}$), which is precisely what yields the $e^{-t}$ factor.

A useful viewpoint of this result, is that the bound is dimension-free and does not invoke any convexity, spectral gap, or log-Sobolev structure. In particular, it provides an $L^\infty$-type control on the density ratio $\rho_t/\rho_{t}^{'}$, which is stronger than controlling an integral divergence (KL/R\'enyi) alone, and can therefore be directly translated into uniform stability statements.
\end{remark}
\begin{remark}
    As $t\to\infty$, the transient term vanishes and the bound saturates at a constant of order $\Vsenbound$/$\Vsenboundalt$ (coming from the intrinsic discrepancy between $\pione$ and $\pitwo$). This term is unavoidable, unless the two target distributions coincide.
\end{remark}
In Theorem \ref{log likelihood and Renyi divergence bound only under potential perturbation bound}, we are able to prove, under very mild assumptions on the initializations of the SHK gradient flow iterates $\rho_0$ and $\rhotwo_0$ corresponding to the potentials $\Vone$ and $\Vtwo$, that the log-likelihood ratio $\log\left(\frac{\rho_t}{\rhotwo_t}\right)$ between the two perturbed flows $\left\{\rho_t\right\}_{t \geq 0}$ and $\left\{\rhotwo_t\right\}_{t \geq 0}$ remain bounded for any $t$, and the bound improves as the flow progresses towards $t \to \infty$. Based on this log-likelihood bound, we are also able to derive a bound on the R\'enyi divergence of orders $\alpha>1$ between $\rho_t$ and $\rhotwo_t$ for any $t \geq 0$. However, the analysis used for Theorem \ref{log likelihood and Renyi divergence bound only under potential perturbation bound} does not extend to the limiting $\alpha \to 1$ scenario, if our goal is to bound the (exclusive) KL divergence between $\rho_t$ and $\rhotwo_t$. To this end, we now proceed to bound the (exclusive) KL divergence between $\rho_t$ and $\rhotwo_t$.

In Theorem \ref{log likelihood and Renyi divergence bound only under potential perturbation bound}, we derived the log-likelihood and R\'enyi divergence bounds under Assumption \ref{Ass: Bound on size of perturbation of potentials}, which bounds the size of perturbation of the potentials in the sense that $\left\|\Vone-\Vtwo\right\|_{\infty} \leq \Vsenbound$, or alternatively under Assumption \ref{Ass: Alternative Bound on size of perturbation of potentials}, which controls the oscillation of the potential sensitivity using the bound $\sup _{\theta \in \Theta}\left(\Vone(\theta)-\Vtwo(\theta)\right) -  \inf _{\theta \in \Theta}\left(\Vone(\theta)-\Vtwo(\theta)\right) \leq \Vsenboundalt $. In order to derive the bound on  $\KLf{\rho_t}{\rhotwo_t}$ and $\rhotwo_t$, we have to rely on an additional assumption involving control of the size of perturbation of the gradients of the potentials i.e. $\left\|\nabla \Vone-\nabla \Vtwo\right\|_{\infty} \leq  \Vgradsenbound$. However, that leads to a bound on $\KLf{\rho_t}{\rhotwo_t}$ which grows atleast linearly in $t$, and hence represents uncontrolled divergence between the two perturbed flows. We have derived this result as Theorem \ref{Time derivative of KL divergence between perturbed flows under gradient sensitivity assumption}, which has been presented in Section \ref{sec: Bound on KL divergence between perturbed flows under control of potentials and their gradients} of the Appendix. Upon further exploration, it was observed that some additional structural assumptions regarding only the target distributions $\pione$ and $\pitwo$ (Log-Sobolev conditions), which are quite standard in the literature, it is possible to derive a bound on $\KLf{\rho_t}{\rhotwo_t}$ that remains bounded and controlled for any time $t \geq 0$ and stabilizes within a short period from initialization. Theorem \ref{Time derivative of KL divergence between perturbed flows under LSI assumption at every timepoint} provides the derived bound.

\begin{theorem}\label{Time derivative of KL divergence between perturbed flows under LSI assumption at every timepoint}
    Assume that we fix a common initial density $\rho_0$ on $\Theta$ such that $\left\{\rho_t\right\}_{t \geq 0}, \left\{\rhotwo_t\right\}_{t \geq 0} \in C^{1,2}([0, T] \times \Theta)$ are strictly positive solutions of the transport-reaction equations \eqref{eq: SHK FPE for potential V} and \eqref{eq: SHK FPE for potential Vprime} on $[0, T]$. 
    Further, let Assumption \ref{Ass: Bound on size of perturbation of gradient of potentials} hold true for the potential functions $\Vone$ and $\Vtwo$. Finally, let $\pione$ and $\pitwo$ satisfy Assumption \ref{Ass: LSI for targets}. Further, assume that $R_{0},R_{0}^{\prime} \leq B$ for some universal constant $B \geq 0$, where 
    \[
    \begin{aligned}
    R_t 
    &:= \max_{\theta \in \Theta} 
    \log\!\left(\frac{\rho_t}{\pi}\right)(\theta)
    - \min_{\theta \in \Theta} 
    \log\!\left(\frac{\rho_t}{\pi}\right)(\theta), \\
    R_t' 
    &:= \max_{\theta \in \Theta} 
    \log\!\left(\frac{\rho_t'}{\pi'}\right)(\theta)
    - \min_{\theta \in \Theta} 
    \log\!\left(\frac{\rho_t'}{\pi'}\right)(\theta)
    \end{aligned}
    \]
and let $\tarloglikbound$ be a constant such that
$\left|\log \frac{\pione(\theta)}{\pitwo(\theta)}\right| \leq \tarloglikbound$ for all $\theta \in \Theta$.

    Then,  for all $t \geq 0$ and any $0 < c < 1$,
    \[
\begin{aligned}
H'(t) 
&:= \frac{d}{dt} \KLf{\rho_t}{\rho_t'}\\
\le &
- 2(1-c)e^{-e^{-t}B}\lambdapi\, H(t) 
+ \frac{1}{4c}\Vgradsenbound^2
+ \frac{(R_0^{\prime})^2}{16} e^{-2t} 
+ \frac{1}{2}\tarloglikbound R_0 e^{-t}
+ 2R_0' e^{-t}.
\end{aligned}
\]

    Consequently, we have that
    \[
\begin{aligned}
H(t)
&= \KLf{\rho_t}{\rho_t'} \\
&\le
\int_0^t
\exp\!\left(
-2(1-c)\lambdapi
\int_s^t
e^{-e^{-u}B}\, du
\right)
\, b_c(s)\, ds,
\end{aligned}
\]
    where $b_c(t) = \frac{1}{4c} \Vgradsenbound^2+\frac{(R_0^{\prime})^2}{16}\exp(-2t)+\frac{1}{2}\tarloglikbound R_0 e^{-t}+2R_0^{\prime}\exp(-t)$.
    
    Further, as a corollary, the following explicit result holds. Let $R_{\Done},R_{\Dtwo} \leq B$ for some $B \geq 0$ independent of the datasets $\Done$ and $\Dtwo$. Define 
    \[
\begin{aligned}
A_1 &:= \frac{\Vgradsenbound^2}{4c},
\quad A_2 := B\left(\frac{\tarloglikbound}{2} + 2\right),\quad A_3 := \frac{B^2}{16},
\quad \kappa := 2(1-c)e^{-B}\lambdapi.
\end{aligned}
\]
    Then for all $t \geq 0$, we have that
     \[
H(t) \le
\begin{cases}

\text{if } \kappa \neq 1,2:
\begin{aligned}[t]
A_1 \frac{1-e^{-\kappa t}}{\kappa}
&+ A_2 \frac{e^{-t}-e^{-\kappa t}}{\kappa-1} \\
&+ A_3 \frac{e^{-2t}-e^{-\kappa t}}{\kappa-2}
\end{aligned}
\\[1.2em]

\text{if } \kappa = 1:
\begin{aligned}[t]
A_1 \frac{1-e^{-\kappa t}}{\kappa}
&+ A_2 t e^{-t} \\
&+ A_3 \frac{e^{-2t}-e^{-\kappa t}}{\kappa-2}
\end{aligned}
\\[1.2em]

\text{if } \kappa = 2:
\begin{aligned}[t]
A_1 \frac{1-e^{-\kappa t}}{\kappa}
&+ A_2 \frac{e^{-t}-e^{-\kappa t}}{\kappa-1} \\
&+ A_3 t e^{-2t}
\end{aligned}

\end{cases} .
\]
Under Assumption \ref{Ass: Bound on size of perturbation of potentials}, the choice $\tarloglikbound = 2 \Vsenbound$ holds, while the choice $\tarloglikbound = \Vsenboundalt$ is valid under Assumption \ref{Ass: Alternative Bound on size of perturbation of potentials}.
\end{theorem}
The proof of Theorem \ref{Time derivative of KL divergence between perturbed flows under LSI assumption at every timepoint} is available in Section \ref{Proof: Time derivative of KL divergence between perturbed flows under LSI assumption at every timepoint} of the Appendix.

\begin{remark}
   Theorem \ref{Time derivative of KL divergence between perturbed flows under LSI assumption at every timepoint} refines Theorem \ref{Time derivative of KL divergence between perturbed flows under gradient sensitivity assumption} by exploiting the log-Sobolev structure of the targets (Assumption \ref{Ass: LSI for targets}) to introduce an explicit negative drift term in the KL dynamics-
\[
H'(t)\ \le\ -c(t)\,H(t) + b_c(t).
\]
Here, $c(t)=2(1-c)\exp(-e^{-t}B)\lambdapi$ arises by lower bounding the Fisher information $I(t)$ in terms of $H(t)$ through LSI, while accounting for the time-dependent Holley-Strook degradation induced by the oscillation bound $R_t\le e^{-t}B$.
This is the main mechanism turning the stability estimate from potentially growing (as in Theorem\ref{Time derivative of KL divergence between perturbed flows under gradient sensitivity assumption}) into uniformly controlled over time.

The factor $\exp(-e^{-t}B)$ captures that the effective LSI constant along the flow and it improves with time as the oscillation of $\log(\rho_t/\pi)$ shrinks.
At $t=0$ this factor is $\exp(-B)$, while as $t\to\infty$ it tends to $1$- thus the contractivity strengthens over time.

The explicit bound implies
\[
H(t)\xrightarrow[t\to\infty]{}\frac{A_1}{\kappa}
\quad\text{(up to some vanishing terms),}
\]
so the KL divergence remains bounded for all $t$ under LSI. In particular, the long-time plateau is controlled by the balance between the forcing size $A_1\propto \Vgradsenbound^2$ and the contractivity rate $\kappa\propto \lambda_{\mathrm{Gibbs}}$.
When $\Vgradsenbound=0$, the forcing size disappears ($A_1=0$) and the bound predicts $H(t)\to 0$, i.e., asymptotic coincidence of the two flows started from the same initialization.
\end{remark}

\section{ Application to Differential Privacy} \label{sec: Application to Differential Privacy}
In this section, we try to apply the previous results in differential privacy. Differential privacy formalizes privacy as a distributional notion of stability for randomized mechanisms operating on different datasets. A mechanism is said to be differentially private if its output measures corresponding to neighboring datasets remain uniformly close. This ensures that no single data point can significantly influence the released distribution. Two datasets $\Done$ and $\Dtwo$ are called neighboring if they differ in the data of at most one individual.
\begin{definition}[Approximate and Pure Differential Privacy]\label{Approximate and Pure DP definition}\citep{dwork2014algorithmic}
    Let $\mathcal{D}$ be the dataset space with a symmetric neighboring/adjacency relation $\Done \sim \Dtwo$. A randomized mechanism $M$ with outputs in $\Theta$ is $(\varepsilon, \delta)$-DP if for all neighboring datasets $\Done$ and $\Dtwo$ and all measurable sets $S \subseteq \Theta$,
    $$
    \mathbb{P}(M(D) \in S) \leq \exp(\varepsilon) \mathbb{P}\left(M\left(D^{\prime}\right) \in S\right) + \delta.
    $$
    If the randomized algorithm $M$ satisfies $(\varepsilon,\delta)$-DP with $\delta=0$, it is said to satisfy $\varepsilon$-Pure DP.
\end{definition}
\begin{theorem}[Pure DP guarantees under sensitivity bound on potentials]\label{Pure DP bound under sensitivity bound on potentials}
    Let $D$ and $\Dtwo$ be neighboring datasets, with potentials $\VDone, \VDtwo$ corresponding to Exponential mechanism densities $\piDone \propto \exp(-\VDone)$ and $\piDtwo \propto \exp(-\VDtwo)$. Fix initial densities $\rho_0^{\Done}$ and $\rho_0^{\Dtwo}$ on $\Theta$ that are independent of the dataset, strictly positive, and smooth enough so that the SHK Fokker Planck Equations \ref{eq: SHK FPE for potential V} and \eqref{eq: SHK FPE for potential Vprime} have positive classical solutions $\rho_t^\Done,$ and $\rho_t^{\Dtwo}$, respectively, for all $t \geq 0$.
    Define
    $$
    s_0^{\mathrm{Data}}(\theta)\coloneqq\log \frac{\rho_0^{\textrm{Data}}(\theta)}{\pi_{{\textrm{Data}}}(\theta)}, 
    $$
    and
    $$
    R_{{\mathrm{Data}}}\coloneqq\max _{\theta \in \Theta}s_0^{{\mathrm{Data}}}(\theta) - \min _{\theta \in \Theta}s_0^{{\mathrm{Data}}}(\theta)
    $$
    for $\mathrm{Data} \in \left\{\Done,\Dtwo\right\}$.
    
    Then, if $R_{\Done},R_{\Dtwo} \leq B$ for some $B \geq 0$ independent of the datasets $\Done$ and $\Dtwo$, and if the potentials $\VDone$ and $\VDtwo$ satisfy Assumption \ref{Ass: Bound on size of perturbation of potentials} i.e. $\left\|\VDone-\VDtwo\right\|_{\infty} \leq \Vsenbound$, then the mechanism $M_t(\Done)$ that outputs $\theta_t \sim \rho_t^\Done$ is $(\varepsilon_1(t), 0)$-DP (i.e. satisfies $\varepsilon_1(t)$-Pure DP) with
    $$
    \varepsilon_1(t)=2 \Vsenbound+2B\exp(-t) .
    $$
    If instead the potentials $\VDone$ and $\VDtwo$ satisfy Assumption \ref{Ass: Alternative Bound on size of perturbation of potentials} i.e. $\sup _{\theta \in \Theta}\left(\VDone(\theta)-\VDtwo(\theta)\right) -  \inf _{\theta \in \Theta}\left(\VDone(\theta)-\VDtwo(\theta)\right) \leq \Vsenboundalt$, then the mechanism $M_t(\Done)$ that outputs $\theta_t \sim \rho_t^\Done$ is $(\varepsilon_2(t), 0)$-DP (i.e. satisfies $\varepsilon_2(t)$-Pure DP) with
    $$
    \varepsilon_2(t)= \Vsenboundalt+2B\exp(-t) .
    $$
\end{theorem}
The proof of Theorem \ref{Pure DP bound under sensitivity bound on potentials} is available in section \ref{Proof: Pure DP bound under sensitivity bound on potentials} in the Appendix. 
\begin{remark}
    Theorem \ref{Pure DP bound under sensitivity bound on potentials} is an immediate (and sharp in spirit) privacy corollary of the pointwise control in Theorem \ref{log likelihood and Renyi divergence bound only under potential perturbation bound}. The key point is that pure DP is governed by uniform likelihood-ratio bounds, i.e. inequalities of the form $\rho_t^D(\theta)\le e^{\varepsilon(t)}\rho_t^{\Dtwo}(\theta)$ for all $\theta$. Theorem \ref{Pure DP bound under sensitivity bound on potentials} shows that, provided the initialization is dataset-independent and the initial log-density mismatch is uniformly bounded ($ R_{\Done},R_{\Dtwo} \leq B$ for some $B \geq 0$), the SHK flow inherits a pure-DP guarantee with an explicit time profile $\varepsilon(t)$ with $\varepsilon(t) = \varepsilon_1(t)$ under Assumption \ref{Ass: Bound on size of perturbation of potentials} and $\varepsilon(t) = \varepsilon_2(t)$ under Assumption \ref{Ass: Alternative Bound on size of perturbation of potentials}, respectively.
\end{remark}
\begin{remark}
    The $\varepsilon(t)$ term consists of: (i) a dataset sensitivity term $ 2 \Vsenbound$ or $\Vsenboundalt$ (the inherent stability scale of the exponential-mechanism targets under Assumption \ref{Ass: Bound on size of perturbation of potentials} or Assumption \ref{Ass: Alternative Bound on size of perturbation of potentials}, respectively), together with (ii) an exponentially decaying transient term $2B\exp(-t)$ stemming only from initialization. Thus, the price of not starting at stationarity manifests as an additive, time-vanishing privacy overhead. As $t\to\infty$, $\varepsilon(t)=\varepsilon_1(t)\to 2\Vsenbound$ under Assumption \ref{Ass: Bound on size of perturbation of gradient of potentials} and $\varepsilon(t)=\varepsilon_2(t)\to \Vsenboundalt$ under Assumption \ref{Ass: Alternative Bound on size of perturbation of potentials}. In other words, once the transient term vanishes, the sampler's pure-DP level matches (up to this sensitivity constant) the irreducible privacy level dictated by how different the neighboring Gibbs targets can be under Assumption \ref{Ass: Bound on size of perturbation of potentials} or \ref{Ass: Alternative Bound on size of perturbation of potentials}. Running the sampler longer cannot improve the pure-DP constant below this level, since even exact samples from $\piDone$ and $\piDtwo$ can differ by that amount in worst-case likelihood ratio.
\end{remark}
Next we show that under stronger assumptions, such as the LSI conditions for the Exponential mechanism, we have the following result.

\begin{theorem}[Approximate DP guarantees under LSI condition]\label{Approximate DP bound under LSI condition}
    Assume that we fix a common initial density $\rho_0$ on $\Theta$ such that $\left\{\rho_t^{\Done}\right\}_{t \geq 0}, \left\{\rho_t^{\Dtwo}\right\}_{t \geq 0} \in C^{1,2}([0, T] \times \Theta)$ are strictly positive solutions of the transport-reaction equations \eqref{eq: SHK FPE for potential V} and \eqref{eq: SHK FPE for potential Vprime} on $[0, T]$. Further, let Assumption \ref{Ass: Bound on size of perturbation of gradient of potentials} hold true for the potential functions $\VDone$ and $\VDtwo$. Finally, let $\piDone$ and $\piDtwo$ satisfy Assumption \ref{Ass: LSI for targets}. Define $R_{{\mathrm{Data}}}\coloneqq\max _{\theta \in \Theta}\log\left(\frac{\rho_0^{\mathrm{Data}}}{\pi_{\mathrm{Data}}}\right)(\theta) - \min _{\theta \in \Theta}\log\left(\frac{\rho_0^{\mathrm{Data}}}{\pi_{\mathrm{Data}}}\right)(\theta)$
    for $\mathrm{Data} \in \left\{\Done,\Dtwo\right\}$ and let $\tarloglikbounddata$ be a constant such that $\left|\log \frac{\piDone(\theta)}{\piDtwo(\theta)}\right| \leq \tarloglikbounddata$ for all $\theta \in \Theta$.
    
    Further, assume that $R_{\Done},R_{\Dtwo} \leq B$ for some $B \geq 0$ independent of the datasets $\Done$ and $\Dtwo$. Define 
    \[
\begin{aligned}
A_1 &\coloneqq \frac{\Vgradsenbound^2}{4c}, 
\quad A_2 \coloneqq B\left(\frac{\tarloglikbounddata}{2} + 2\right), \quad
A_3 \coloneqq \frac{B^2}{16}, 
\quad \kappa \coloneqq 2(1-c)\exp(-B)\lambdapi.
\end{aligned}
\]
    
    Then, for all $t \geq 0$ and any $0 < c < 1$, we can choose any $\varepsilon>0$ such that the mechanism $M_t(\Done)$ that outputs $\theta_t \sim \rho_t^\Done$ is $(\varepsilon, \delta_t)$-DP (i.e. satisfies $(\varepsilon, \delta_t)$ Approximate-DP) with
    $$
    \delta_t = \frac{\bar{H}(t)}{\varepsilon}.
    $$
    where, 
    \begin{equation*}
\bar{H}(t)=
\begin{cases}

\text{if } \kappa \neq 1,2:
\begin{aligned}[t]
A_1 \frac{1-e^{-\kappa t}}{\kappa}
&+ A_2 \frac{e^{-t}-e^{-\kappa t}}{\kappa-1} \\
&+ A_3 \frac{e^{-2t}-e^{-\kappa t}}{\kappa-2}
\end{aligned}
\\[1.2em]

\text{if } \kappa = 1:
\begin{aligned}[t]
A_1 \frac{1-e^{-\kappa t}}{\kappa}
&+ A_2 t e^{-t} \\
&+ A_3 \frac{e^{-2t}-e^{-\kappa t}}{\kappa-2}
\end{aligned}
\\[1.2em]

\text{if } \kappa = 2:
\begin{aligned}[t]
A_1 \frac{1-e^{-\kappa t}}{\kappa}
&+ A_2 \frac{e^{-t}-e^{-\kappa t}}{\kappa-1} \\
&+ A_3 t e^{-2t}
\end{aligned}

\end{cases} .
\end{equation*}
Under Assumption \ref{Ass: Bound on size of perturbation of potentials}, the choice $\tarloglikbounddata = 2 \Vsenbound$ holds, while the choice $\tarloglikbounddata = \Vsenboundalt$ is valid under Assumption \ref{Ass: Alternative Bound on size of perturbation of potentials}.
\end{theorem}
The detailed proof is provided in section \ref{Proof: Approximate DP bound under LSI condition} of the Appendix.

\begin{remark}
    Theorem \ref{Approximate DP bound under LSI condition} moves beyond pure DP by leveraging a $\KL$-based route to $(\varepsilon,\delta)$-DP. Concretely, it combines -
(i) a non-asymptotic bound on $H(t)=\mathrm{KL}(\rho_t^{\Done}\|\rho_t^{\Dtwo})$(derived from Theorem \ref{Time derivative of KL divergence between perturbed flows under LSI assumption at every timepoint}), and
(ii) the standard inequality relating hockey-stick divergence to KL,which yields $(\varepsilon,\delta)$-DP with $\delta$ proportional to $\frac{H(t)}{\varepsilon}$. The explicit upper bound $\bar H(t)$ reveals three distinct contributions:
\begin{itemize}
\item $A_1$ captures the effect of \emph{gradient sensitivity} $\Vgradsenbound$ and comes from the drift mismatch term; it is the only term that does not come with an explicit $e^{-t}$ or $e^{-2t}$ prefactor.
\item $A_2$ and $A_3$ encode transient contributions driven by the potential sensitivity $\Vsenbound$/$\Vsenboundalt$ and the initialization oscillation bound $B$; these terms decay at rates $e^{-t}$ and $e^{-2t}$ (depending on $\kappa$).
\item The parameter $c\in(0,1)$ is a tuning knob originating from Young's inequality: decreasing $c$ strengthens the contraction rate $\kappa$ (through $1-c$) but increases $A_1=\Vgradsenbound^2/(4c)$. Hence, $c$ mediates an explicit tradeoff between contractivity and the size of the forcing term.
\end{itemize}
\end{remark}
\begin{remark}
    For any $\kappa>0$, the transient terms vanish and
\[
\bar H(t)\xrightarrow[t\to\infty]{}\frac{A_1}{\kappa}
\quad\text{and therefore}\quad
\delta_t=\frac{\bar H(t)}{\varepsilon}\xrightarrow[t\to\infty]{}\frac{A_1}{\kappa\,\varepsilon}.
\]
 Thus, the approximate-DP guarantee saturates to a nonzero plateau whenever $\Vgradsenbound>0$.
Intuitively, as the sampler runs longer it becomes more dataset-dependent, until the stability is limited by the balance between (i) the contractive effect induced by the LSI and (ii) the persistent forcing due to the drift perturbation. In the special case $\Vgradsenbound=0$ (no drift mismatch), one has $A_1=0$ and the bound predicts $\delta_t\to 0$, and we approach pure-DP guarantees in the long-time regime.
\end{remark}

\subsection{Utility analysis}

We now proceed to analyze the utility of the marginal densities $\left\{\rho_t\right\}_{t \geq 0}$ generated by the SHK based sampler. In the standard setting when the Exponential mechanism, introduced by \cite{exponentialmech},  is typically used, the goal is to obtain values $\theta \in \Theta$ which are sufficiently close to being minimizers of the the objective function of interest, say $f: \Theta \to \R$, which is any measurable function bounded below. Let$
f_*:=\inf _{\theta \in \Theta} f(\theta)$.

Choosing any $\beta >0$, the exponential-mechanism target density is constructed as
$$
\pi_\beta(\theta):=\frac{\exp(-\beta f(\theta))}{Z_\beta} , \quad Z_\beta:=\int_{\Theta} \exp(-\beta f(\theta)) d \theta.
$$
In the most common scenarios, the function $f$ depends on sensitive/confidential information $D$, which is typically a dataset belonging to a dataset space $\mathcal{D}$. The standard practice in Differential Privacy is to choose $\beta \propto \frac{\varepsilon}{\Delta}$ where $\Delta$ is a measure of how sensitive the function $f$ is to perturbation of the dataset. For example, a dataset could be a sample of $n$ i.i.d individuals from a population, and the sensitivity $\Delta$ of the function $f$ characterizes the amount of perturbation of the function when a single individual's data is changed, creating \say{neighboring} datasets. Lower $\beta$ will lead to more uniformly distributed mechanisms, leading to better privacy guarantees. For ease of presentation, we will ignore the implications of the choice of $\beta$ in this work, and focus on the utility analysis. 

In order to proceed with the analysis, let us define the volume of a near-optimal sublevel set. For any $\alpha>0$, define $S_\alpha 
:= \{ x \in \Theta : f(x) \le f_* + \alpha \}$, $m_\alpha 
:= \operatorname{Vol}(S_\alpha)$ and $m 
:= \operatorname{Vol}(\Theta)$, where all volumes are with respect to the Lebesgue measure. We first characterize how close is the expected value of the objective function $f$ under the exponential mechanism to the true minimum of the function $f$. 

\begin{proposition}[Expected suboptimality of the exponential mechanism]\label{Expected suboptimality of the exponential mechanism}
Assume $0<m_\alpha \leq m<\infty$. Then
\[
\E_{\pi_\beta}[f]-f_* \leq \alpha+\frac{1}{\beta}\left(\log \frac{m}{m_\alpha}+1\right) .
\]
\end{proposition}

Let $\rho_t$ be the law of the output of the SHK gradient flow based sampler at time $t$. The goal is to control
$\E_{\rho_t}[f]-f_*$. We characterize the closeness of the expected value of the objective function $f$ under the law $\rho_t$ at time $t$ to the true minimum of the function $f$ in the following theorem.

\begin{theorem}[Utility bound for SHK gradient flow sampler]\label{Utility bound for SHK gradient flow sampler} 
    Let $\Theta$ be a compact domain with volume $m$. Let $f: \Theta \to \mathbb{R}$ be measurable, bounded, and bounded below, with $f_*=\inf _{\Theta} f$. Fix $\beta>0$ and let $\pi_\beta \propto \exp(- \beta f)$  be the density of the exponential mechanism.

Let $\rho_t$ be the marginal density of the WFR/SHK sampler targeting $\pi_\beta$ at time $t$ constructed by choosing the potential function $V = \beta f$. Assume there exists a time $t_0$ such that $H\left(t_0\right):=\mathrm{KL}\left(\rho_{t_0} \| \pi_\beta\right) \leq 1$ and $$\inf _{\theta \in \Theta}\frac{ \rho_{t_0}(\theta)}{\pi_\beta(\theta)} \geq \exp(-M)$$ for some $M \geq 1$. Fix $\delta \in(0,1 / 4)$ and define $t_*=t_0+\log (M / \delta)+ \frac{1}{1-2 \delta} \log (1 / \delta)$. Then for every $t \geq t_*$ and every $\alpha>0$ with $m_\alpha:=\operatorname{Vol}\left(S_\alpha\right)>0$, we have that
$$
\begin{aligned}
    \E_{\rho_t}[f]-f_* \leq \alpha+\frac{1}{\beta}\left(\log \frac{m}{m_\alpha}+1\right)+2\|f\|_{\infty} \sqrt{\frac{1}{2} H\left(t_0\right)} \exp \left(-\frac{2-3 \delta}{2}\left(t-t_*\right)\right) .
\end{aligned}
$$
\end{theorem}
The proofs of Proposition \ref{Expected suboptimality of the exponential mechanism} and Theorem \ref{Utility bound for SHK gradient flow sampler} are available in Section \ref{Proof: Expected suboptimality of the exponential mechanism} and Section \ref{Proof: Utility bound for SHK gradient flow sampler} of the Appendix, respectively.
\section{Conclusion}

This work analyzes the sensitivity of SHK/WFR gradient-flow samplers, equivalently, birth-death accelerated Langevin dynamics, and relates it to perturbations of the underlying potential. The central contribution is a nonasymptotic perturbation bound for the time-marginal laws generated by two SHK flows driven by $\Vone$ and $\Vtwo$, started from the same initial density. Then, these stability statements are translated directly into differential privacy guarantees for exponential-mechanism sampling, which has practical applications in diverse domains. Existing papers on SHK gradient flows such as \citet{lu2019accelerating} already convincingly establish the advantages of the additional flexibility and acceleration gained due to the reaction term in comparison to Kantorovich (Wasserstein) transport-based Langevin dynamics. We perform some toy numerical experiments to test the validity of our perturbation bounds and empirically demonstrate the stability of perturbed SHK flows in Section \ref{sec: Experiments} of the Appendix. As a natural extension to this work, one can try to relax the assumptions that we make and also develop R\'enyi-DP or $f$-DP based guarantees for SHK gradient flow based samplers.

\section*{References}
\bibliographystyle{plainnat}
\bibliography{ref.bib}

@INPROCEEDINGS{exponentialmech,  author={McSherry, Frank and Talwar, Kunal},  booktitle={48th Annual IEEE Symposium on Foundations of Computer Science (FOCS'07)},   title={Mechanism Design via Differential Privacy},   year={2007},  volume={},  number={},  pages={94-103},  doi={10.1109/FOCS.2007.66}}

@article{lu2019accelerating,
  title={Accelerating Langevin sampling with birth-death},
  author={Lu, Yulong and Lu, Jianfeng and Nolen, James},
  journal={arXiv preprint arXiv:1905.09863},
  year={2019}
}

@article{borovykh2023privacy,
  title={Privacy Risk for anisotropic Langevin dynamics using relative entropy bounds},
  author={Borovykh, Anastasia and Kantas, Nikolas and Parpas, Panos and Pavliotis, Greg},
  journal={arXiv preprint arXiv:2302.00766},
  year={2023}
}

@article{jordan1998variational,
  title={The variational formulation of the Fokker--Planck equation},
  author={Jordan, Richard and Kinderlehrer, David and Otto, Felix},
  journal={SIAM journal on mathematical analysis},
  volume={29},
  number={1},
  pages={1--17},
  year={1998},
  publisher={SIAM}
}

@article{doi:10.1137/S0036141096303359,
author = {Jordan, Richard and Kinderlehrer, David and Otto, Felix},
title = {The Variational Formulation of the Fokker--Planck Equation},
journal = {SIAM Journal on Mathematical Analysis},
volume = {29},
number = {1},
pages = {1-17},
year = {1998},
doi = {10.1137/S0036141096303359},

URL = { 
    
        https://doi.org/10.1137/S0036141096303359
    
    

},
eprint = { 
    
        https://doi.org/10.1137/S0036141096303359
    
    

}
}

@article{schlichting2019poincare,
  title={Poincar{\'e} and log--sobolev inequalities for mixtures},
  author={Schlichting, Andr{\'e}},
  journal={Entropy},
  volume={21},
  number={1},
  pages={89},
  year={2019},
  publisher={MDPI}
}

@article{altschuler2022privacy,
  title={Privacy of noisy stochastic gradient descent: More iterations without more privacy loss},
  author={Altschuler, Jason and Talwar, Kunal},
  journal={Advances in Neural Information Processing Systems},
  volume={35},
  pages={3788--3800},
  year={2022}
}

@article{tan2025accelerate,
  title={Accelerate langevin sampling with birth-death process and exploration component},
  author={Tan, Lezhi and Lu, Jianfeng},
  journal={SIAM/ASA Journal on Uncertainty Quantification},
  volume={13},
  number={3},
  pages={1265--1293},
  year={2025},
  publisher={SIAM}
}

@article{kondratyev2019spherical,
  title={Spherical Hellinger--Kantorovich Gradient Flows},
  author={Kondratyev, Stanislav and Vorotnikov, Dmitry},
  journal={SIAM Journal on Mathematical Analysis},
  volume={51},
  number={3},
  pages={2053--2084},
  year={2019},
  publisher={SIAM}
}

@article{holley1986logarithmic,
  title={Logarithmic Sobolev inequalities and stochastic Ising models},
  author={Holley, Richard and Stroock, Daniel W},
  year={1986},
  publisher={Laboratory for Information and Decision Systems, Massachusetts Institute of~…}
}

@article{liero2018optimal,
  title={Optimal entropy-transport problems and a new Hellinger--Kantorovich distance between positive measures},
  author={Liero, Matthias and Mielke, Alexander and Savar{\'e}, Giuseppe},
  journal={Inventiones mathematicae},
  volume={211},
  number={3},
  pages={969--1117},
  year={2018},
  publisher={Springer}
}

@article{chizat2015interpolating,
  title={An interpolating distance between optimal transport and Fisher-Rao},
  author={Chizat, Lenaic and Schmitzer, Bernhard and Peyr{\'e}, Gabriel and Vialard, Fran{\c{c}}ois-Xavier},
  journal={arXiv preprint arXiv:1506.06430},
  year={2015}
}

@article{sason2016hockeystickdivergence,
  title={$ f $-divergence Inequalities},
  author={Sason, Igal and Verd{\'u}, Sergio},
  journal={IEEE Transactions on Information Theory},
  volume={62},
  number={11},
  pages={5973--6006},
  year={2016},
  publisher={IEEE}
}

@article{dwork2014algorithmic,
  title={The algorithmic foundations of differential privacy},
  author={Dwork, Cynthia and Roth, Aaron},
  journal={Foundations and trends{\textregistered} in theoretical computer science},
  volume={9},
  number={3-4},
  pages={211--487},
  year={2014},
  publisher={Emerald Publishing Limited}
}

@article{balle2018privacy,
  title={Privacy amplification by subsampling: Tight analyses via couplings and divergences},
  author={Balle, Borja and Barthe, Gilles and Gaboardi, Marco},
  journal={Advances in neural information processing systems},
  volume={31},
  year={2018}
}

@article{minasyan2018hockey,
  title={Hockey-Stick GAN},
  author={Minasyan, Edgar and Prabhu, Vinay},
  year={2018}
}

@inproceedings{abadi2016deep,
  title={Deep learning with differential privacy},
  author={Abadi, Martin and Chu, Andy and Goodfellow, Ian and McMahan, H Brendan and Mironov, Ilya and Talwar, Kunal and Zhang, Li},
  booktitle={Proceedings of the 2016 ACM SIGSAC conference on computer and communications security},
  pages={308--318},
  year={2016}
}

@article{chourasia2021differential,
  title={Differential privacy dynamics of langevin diffusion and noisy gradient descent},
  author={Chourasia, Rishav and Ye, Jiayuan and Shokri, Reza},
  journal={Advances in Neural Information Processing Systems},
  volume={34},
  pages={14771--14781},
  year={2021}
}

@article{ganesh2020faster,
  title={Faster differentially private samplers via R{\'e}nyi divergence analysis of discretized Langevin MCMC},
  author={Ganesh, Arun and Talwar, Kunal},
  journal={Advances in Neural Information Processing Systems},
  volume={33},
  pages={7222--7233},
  year={2020}
}

@article{ganesh2022differentially,
  title={Differentially private sampling from rashomon sets, and the universality of langevin diffusion for convex optimization},
  author={Ganesh, Arun and Thakurta, Abhradeep and Upadhyay, Jalaj},
  journal={arXiv preprint arXiv:2204.01585},
  year={2022}
}

@inproceedings{wibisono2017information,
  title={Information and estimation in Fokker-Planck channels},
  author={Wibisono, Andre and Jog, Varun and Loh, Po-Ling},
  booktitle={2017 IEEE International Symposium on Information Theory (ISIT)},
  pages={2673--2677},
  year={2017},
  organization={IEEE}
}

@article{mahloujifar2022optimal,
  title={Optimal membership inference bounds for adaptive composition of sampled gaussian mechanisms},
  author={Mahloujifar, Saeed and Sablayrolles, Alexandre and Cormode, Graham and Jha, Somesh},
  journal={arXiv preprint arXiv:2204.06106},
  year={2022}
}

@inproceedings{melis2019exploiting,
  title={Exploiting unintended feature leakage in collaborative learning},
  author={Melis, Luca and Song, Congzheng and De Cristofaro, Emiliano and Shmatikov, Vitaly},
  booktitle={2019 IEEE symposium on security and privacy (SP)},
  pages={691--706},
  year={2019},
  organization={IEEE}
}

@article{bogachev2014kantorovich,
  title={The Kantorovich and variation distances between invariant measures of diffusions and nonlinear stationary Fokker-Planck-Kolmogorov equations},
  author={Bogachev, Vladimir Igorevich and Kirillov, Andrei Igorevich and Shaposhnikov, Stanislav Valer'evich},
  journal={Mathematical Notes},
  volume={96},
  number={5},
  pages={855--863},
  year={2014},
  publisher={Springer}
}

@article{bogachev2016distances,
title = {Distances between transition probabilities of diffusions and applications to nonlinear Fokker–Planck–Kolmogorov equations},
journal = {Journal of Functional Analysis},
volume = {271},
number = {5},
pages = {1262-1300},
year = {2016},
issn = {0022-1236},
doi = {https://doi.org/10.1016/j.jfa.2016.05.016},
url = {https://www.sciencedirect.com/science/article/pii/S0022123616301239},
author = {V.I. Bogachev and M. Röckner and S.V. Shaposhnikov}
}

@article{pavliotis2014stochastic,
  title={Stochastic processes and applications},
  author={Pavliotis, Grigorios A},
  journal={Texts in applied mathematics},
  volume={60},
  pages={41--43},
  year={2014},
  publisher={Springer}
}

@article{mattingly2002ergodicity,
title = {Ergodicity for SDEs and approximations: locally Lipschitz vector fields and degenerate noise},
journal = {Stochastic Processes and their Applications},
volume = {101},
number = {2},
pages = {185-232},
year = {2002},
issn = {0304-4149},
doi = {https://doi.org/10.1016/S0304-4149(02)00150-3},
url = {https://www.sciencedirect.com/science/article/pii/S0304414902001503},
author = {J.C. Mattingly and A.M. Stuart and D.J. Higham}
}

@article{sanz2016gaussian,
  title={Gaussian approximations of small noise diffusions in Kullback-Leibler divergence},
  author={Sanz-Alonso, Daniel and Stuart, Andrew M},
  journal={arXiv preprint arXiv:1605.05878},
  year={2016}
}

@article{Malrieu_2001, title={Logarithmic Sobolev inequalities for some nonlinear PDE’s}, volume={95}, ISSN={0304-4149}, url={http://dx.doi.org/10.1016/s0304-4149(01)00095-3}, DOI={10.1016/s0304-4149(01)00095-3}, number={1}, journal={Stochastic Processes and their Applications}, publisher={Elsevier BV}, author={Malrieu, F.}, year={2001}, month=sep, pages={109–132} }

@inbook{Ledoux_2001, title={Logarithmic Sobolev Inequalities for Unbounded Spin Systems Revisited}, ISBN={9783540446712}, ISSN={1617-9692}, url={http://dx.doi.org/10.1007/978-3-540-44671-2_13}, DOI={10.1007/978-3-540-44671-2_13}, booktitle={Séminaire de Probabilités XXXV}, publisher={Springer Berlin Heidelberg}, author={Ledoux, M.}, year={2001}, pages={167–194} }

@inbook{Ledoux_1999, title={Concentration of measure and logarithmic Sobolev inequalities}, ISBN={9783540484073}, ISSN={1617-9692}, url={http://dx.doi.org/10.1007/bfb0096511}, DOI={10.1007/bfb0096511}, booktitle={Séminaire de Probabilités XXXIII}, publisher={Springer Berlin Heidelberg}, author={Ledoux, Michel}, year={1999}, pages={120–216} }

@inbook{Bakry_2014, title={Symmetric Markov Diffusion Operators}, ISBN={9783319002279}, ISSN={2196-9701}, url={http://dx.doi.org/10.1007/978-3-319-00227-9_3}, DOI={10.1007/978-3-319-00227-9_3}, booktitle={Analysis and Geometry of Markov Diffusion Operators}, publisher={Springer International Publishing}, author={Bakry, Dominique and Gentil, Ivan and Ledoux, Michel}, year={2014}, pages={119–174} }
\nocite{*}

\newpage

\appendix

\section{Technical appendices and supplementary material}

\subsection{Assumptions on the potential function}\label{sec: Assumptions}

\begin{asu}[Bound on size of perturbation of potentials]\label{Ass: Bound on size of perturbation of potentials}
    Let the potentials $\Vone$ and $\Vtwo$ satisfy a uniform pointwise sensitivity bound:
    $$
    \left\|\Vone-\Vtwo\right\|_{\infty}\coloneqq\sup _{\theta \in \Theta}\left|\Vone(\theta)-\Vtwo(\theta)\right| \leq \Vsenbound .
    $$
\end{asu}

\begin{asuprime}[Alternative bound on size of perturbation of potentials]\label{Ass: Alternative Bound on size of perturbation of potentials}
    Let the potentials $\Vone$ and $\Vtwo$ satisfy a uniform oscillation sensitivity bound:
    $$
    \sup _{\theta \in \Theta}\left(\Vone(\theta)-\Vtwo(\theta)\right) -  \inf _{\theta \in \Theta}\left(\Vone(\theta)-\Vtwo(\theta)\right) \leq \Vsenboundalt .
    $$
\end{asuprime}

\begin{asu}[Bound on size of perturbation of gradient of potentials]\label{Ass: Bound on size of perturbation of gradient of potentials}
    Let the gradients of the potentials $\Vone$ and $\Vtwo$ satisfy a uniform pointwise sensitivity bound:
    $$
    \left\|\nabla \Vone-\nabla \Vtwo\right\|_{\infty}\coloneqq\sup _{\theta \in \Theta}\left\|\nabla\Vone(\theta)-\nabla\Vtwo(\theta)\right\| \leq \Vgradsenbound .
    $$
\end{asu}

\begin{asu}[Log Sobolev Inequality (LSI) for target Gibbs densities]\label{Ass: LSI for targets}
Let $\pione \propto \exp(- \Vone)$ and $\pitwo \propto \exp(- \Vtwo)$ satisfy a log-Sobolev inequality with the same constant $\lambdapi>0$ i.e. for all smooth $f$ with $\int f^2 d \pione=1$ and $\int f^2 d \pitwo=1$, it holds that
$$
\operatorname{Ent}_{\pione}\left(f^2\right)\coloneqq\int f^2 \log f^2 d \pione \leq \frac{2}{\lambdapi} \int\|\nabla f\|^2 d \pione, \quad \textrm{and } \quad\operatorname{Ent}_{\pitwo}\left(f^2\right)\coloneqq\int f^2 \log f^2 d \pitwo \leq \frac{2}{\lambdapi} \int\|\nabla f\|^2 d \pitwo.
$$
\end{asu}

\begin{remark}
    The potentials $\Vone$ and $\Vtwo$ are only defined up to additive constants. Replacing $\Vone$ by $\Vone+c$ does not change the Gibbs density $\pione$ since $\frac{\exp\left[-(V+c)\right]}{\int_{\Theta} \exp\left[-(V+c)\right]}=\frac{\exp\left(-c\right) \times\exp\left(-V\right)}{\exp\left(-c\right) \times \int_{\Theta} \exp\left(-V\right)}=\frac{\exp\left(-V\right)}{ \int_{\Theta} \exp\left(-V\right)}$. The transport term in the SHK dynamics (Equation \ref{eq: SHK dynamics Fokker Planck Equation}) also remains invariant since $\nabla(V+c)=\nabla V$. The reaction term is also unchanged since it depends on $\log \frac{\rho}{\pi}$, and $\pi$ is unchanged. Therefore the quantity $\left\|\Vone-\Vtwo\right\|_{\infty}$ is not intrinsic: it changes if one adds a constant to $V$ or $V^{\prime}$. The intrinsic quantity is instead $\inf _{c \in \R}\left\|\Vone-\Vtwo-c\right\|_{\infty}$, which can be exactly expressed as $\frac{1}{2} \left[\sup _{\theta \in \Theta}\left(\Vone(\theta)-\Vtwo(\theta)\right) -  \inf _{\theta \in \Theta}\left(\Vone(\theta)-\Vtwo(\theta)\right)\right]$. Consequently, considering a uniform oscillation sensitivity bound of the form considered in Assumption \ref{Ass: Alternative Bound on size of perturbation of potentials} may be more natural.
\end{remark}

\begin{remark}
Assumption \ref{Ass: LSI for targets} postulates that the two Gibbs targets $\pi\propto \exp(-\Vone)$ and $\pitwo \propto \exp(-\Vtwo)$ satisfy a log-Sobolev inequality with a common constant $\lambda_{\mathrm{Gibbs}}>0$.
In the differential privacy application (where $V_D=\varepsilon_{\mathrm{priv}}\,L_D$ is the exponential-mechanism potential associated with dataset $D$), it is often reasonable to treat $\lambda_{\mathrm{Gibbs}}$ as uniform over neighboring datasets (or even over all datasets in a class), for several standard reasons:

\begin{itemize}
\item  Uniform strong log-concavity / curvature control-
If there exists $m>0$, independent of $D$, such that $\nabla^2 V_D(\theta)\succeq m I_d$ for all $\theta\in\Theta$ and all datasets $D$, then each $\piDone$ satisfies an LSI with constant at least $m$ (e.g., by Bakry-\'Emery-type criteria).
In exponential-mechanism settings, such uniform curvature is commonly enforced by adding a dataset independent strongly convex regularizer (e.g., $L_D(\theta)=\frac1n\sum_i \ell(\theta;x_i)+\frac{\mu}{2}\|\theta\|^2$), in which case $m$ is governed by the regularization strength and scales linearly with the privacy parameter $\varepsilon_{\mathrm{priv}}$ through $V_D=\varepsilon_{\mathrm{priv}}L_D$.
\item Stability of LSI under bounded perturbations (Holley-Strook - Lemma \ref{Holley Strook Pertubation Lemma}) - 
More generally, suppose a reference Gibbs measure $\bar\pi\propto \exp(- \bar V)$ satisfies an LSI with constant $\bar\lambda>0$, and each dataset-dependent potential decomposes as
$V_D=\bar V+\psi_D$ with uniformly bounded oscillation $\mathrm{osc}(\psi_D)=\sup\psi_D-\inf\psi_D\le C$ (independent of $D$).
Then the Lemma \ref{Holley Strook Pertubation Lemma} implies that $\piDone$ satisfies an LSI with constant at least $\exp(-C)\bar\lambda$.
In particular, for neighboring datasets $D\sim D'$, Assumption \ref{Ass: Bound on size of perturbation of potentials} gives the uniform bound
$\mathrm{osc}(\VDone-\VDtwo)\le 2\Vsenbound$ while Assumption \ref{Ass: Alternative Bound on size of perturbation of potentials} directly gives the bound $\mathrm{osc}(\VDone-\VDtwo)\le \Vsenboundalt$,
so LSI constants for $\piDone$ and $\piDtwo$ are comparable up to a multiplicative factor depending only on $\Vsenbound$/$\Vsenboundalt$.
Thus, it is natural to work with a common constant chosen as a uniform lower bound over the dataset class.
\item Compactness and dataset-independent regularity bounds -
When $\Theta$ is compact (as assumed throughout) and the family of potentials $\{V_D\}_D$ satisfies dataset-uniform regularity/oscillation bounds, one can often obtain a strictly positive LSI constant that depends only on these uniform bounds and the geometry of $\Theta$ (and hence can be taken independent of $D$)
\end{itemize}
Finally, we stress that Assumption \ref{Ass: LSI for targets} is only used to obtain contractive control of KL (via Fisher information) in Theorems \ref{Time derivative of KL divergence between perturbed flows under LSI assumption at every timepoint}, \ref{Time derivative of KL divergence between perturbed flows under gradient sensitivity assumption} and \ref{Approximate DP bound under LSI condition}. The pointwise perturbation bound in Theorem \ref{log likelihood and Renyi divergence bound only under potential perturbation bound}, and the resulting pure DP guarantee in Theorem \ref{Pure DP bound under sensitivity bound on potentials}, do not rely on any LSI assumption.
\end{remark}

\subsection{Proof of theorems}
\subsubsection{Proof of Theorem \ref{log likelihood and Renyi divergence bound only under potential perturbation bound}} \label{Proof: log likelihood and Renyi divergence bound only under potential perturbation bound}
\begin{proof}
    Fix $t \geq 0$ and $\theta \in \Theta$. Then, we have that
$$
\rho_t(\theta)=\pione(\theta) u_t(\theta), \quad \rhotwo_t(\theta)=\pitwo(\theta) u_t^{\prime}(\theta),
$$
where $u_t=\frac{\rho_t}{\pione}, u_t^{\prime}=\frac{\rhotwo_t}{\pitwo}$.
Then
$$
\log \frac{\rho_t(\theta)}{\rhotwo_t(\theta)}=\log \frac{\pione(\theta)}{\pitwo(\theta)}+\log \frac{u_t(\theta)}{u_t^{\prime}(\theta)} .
$$
Taking absolute values and using triangle inequality, we have that
$$
\left|\log \frac{\rho_t(\theta)}{\rhotwo_t(\theta)}\right| \leq\left|\log \frac{\pione(\theta)}{\pitwo(\theta)}\right|+\left|\log u_t(\theta)-\log u_t^{\prime}(\theta)\right| \leq\left|\log \frac{\pione(\theta)}{\pitwo(\theta)}\right| + \left|\log u_t(\theta)\right|+\left|\log u_t^{\prime}(\theta)\right|.
$$
Using Proposition \ref{Bound on log-likelihood of density at time t to target crude bound}, we have that
$\left|\log u_t(\theta)\right| \leq \exp(-t) R_0$ and $\left|\log u_t^{\prime}(\theta)\right| \leq \exp(-t) R_0^{\prime}$.
Let $\tarloglikbound$ be any constant such that
$$
\left|\log \frac{\pione(\theta)}{\pitwo(\theta)}\right| \leq \tarloglikbound \quad \text{for all } \theta \in \Theta .
$$
Combining, we have that
\begin{equation}\label{loglikelihood bound pointwise in time}
\begin{aligned}
\left|\log \frac{\rho_t(\theta)}{\rhotwo_t(\theta)}\right| &\leq\left|\log \frac{\pione(\theta)}{\pitwo(\theta)}\right| + \left|\log u_t(\theta)\right|+\left|\log u_t^{\prime}(\theta)\right| \\
&\leq \tarloglikbound + \exp(-t)(R_0+R_0^{\prime}).
\end{aligned}
\end{equation}
Under Assumption \ref{Ass: Bound on size of perturbation of potentials}, Lemma \ref{normalization constant and target density sensitivity bound} gives $\tarloglikbound=2\Vsenbound$. Under Assumption \ref{Ass: Alternative Bound on size of perturbation of potentials}, Lemma \ref{target density sensitivity bound under tighter control on oscillation of potentials} gives $\tarloglikbound=\Vsenboundalt$. This proves the two pointwise bounds.

Also, whenever
$$
\left|\log \frac{\rho_t(\theta)}{\rhotwo_t(\theta)}\right|\leq \tarloglikbound + \exp(-t)(R_0+R_0^{\prime}),
$$
we have
$$
\begin{aligned}
    \Renyi{\rho_t}{\rhotwo_t}=&\frac{1}{\alpha-1} \log \int\left(\frac{\rho_t}{\rhotwo_t}\right)^{\alpha-1} d \rho_t\\
    \leq & \frac{1}{\alpha-1} \log \int_{\Theta}\exp\left(\left[\alpha-1\right] \left[\tarloglikbound + \exp(-t)(R_0+R_0^{\prime})\right]\right) d \rho_t\\
    =&\tarloglikbound + \exp(-t)(R_0+R_0^{\prime}).
\end{aligned}
$$
Applying the two choices of $\tarloglikbound$ above completes the proof. 

\end{proof}

\subsubsection{Bound on KL divergence between perturbed flows under $L^{\infty}$ control of potentials and their gradients}\label{sec: Bound on KL divergence between perturbed flows under control of potentials and their gradients}

\begin{theorem}\label{Time derivative of KL divergence between perturbed flows under gradient sensitivity assumption}
    Assume that we fix a common initial density $\rho_0$ on $\Theta$ such that $\left\{\rho_t\right\}_{t \geq 0}, \left\{\rhotwo_t\right\}_{t \geq 0} \in C^{1,2}([0, T] \times \Theta)$ are strictly positive solutions of the transport-reaction equations \eqref{eq: SHK FPE for potential V} and \eqref{eq: SHK FPE for potential Vprime} on $[0, T]$. 
    Finally, let Assumption \ref{Ass: Bound on size of perturbation of gradient of potentials} hold true for the potential functions $\Vone$ and $\Vtwo$ and let $\tarloglikbound$ be a constant such that $\left|\log \frac{\pione(\theta)}{\pitwo(\theta)}\right| \leq \tarloglikbound$ for all $\theta \in \Theta$. Then,  for all $t \geq 0$,
    
    $$
    H^{\prime}(t) \coloneqq \frac{d}{dt}\KLf{\rho_t}{\rhotwo_t} \leq \frac{1}{4} \Vgradsenbound^2+\frac{(R_0^{\prime})^2}{16}\exp(-2t)+\frac{1}{2}\tarloglikbound R_0 \exp(-t)+2R_0^{\prime}\exp(-t)
    $$
    where 
    $$
    R_t\coloneqq\max _{\theta \in \Theta}\log\left(\frac{\rho_t}{\pione}\right)(\theta) - \min _{\theta \in \Theta}\log\left(\frac{\rho_t}{\pione}\right)(\theta), \quad R_t^{\prime}\coloneqq\max _{\theta \in \Theta}\log\left(\frac{\rhotwo_t}{\pitwo}\right)(\theta) - \min _{\theta \in \Theta}\log\left(\frac{\rhotwo_t}{\pitwo}\right)(\theta).
    $$
    Then for all $t \geq 0$,
    $$
        H(t) \leq \frac{1}{4}\Vgradsenbound^2 t + \frac{(R_0^{\prime})^2}{32} (1-\exp(-2t)) + \left(\frac{\tarloglikbound R_0}{2} + 2R_0^{\prime}\right)(1-\exp(-t)).
    $$
    Under Assumption \ref{Ass: Bound on size of perturbation of potentials}, the choice $\tarloglikbound = 2 \Vsenbound$ holds, while the choice $\tarloglikbound = \Vsenboundalt$ is valid under Assumption \ref{Ass: Alternative Bound on size of perturbation of potentials}.
\end{theorem}

\begin{proof}
    Under the given conditions, Proposition \ref{Closed form for time derivative of KL divergence between perturbed SHK flows} gives that
    $$
\begin{aligned}
H^{\prime}(t)= &\frac{d}{dt}\KLf{\rho_t}{\rhotwo_t} \\= & -\int_{\Theta} \left\|\nabla \log\left(\frac{\rho_t}{\rhotwo_t}\right)\right\|^2 d \rho_t+\int_{\Theta} \left(\nabla \Vtwo-\nabla \Vone\right) \cdot \nabla \log \left(\frac{\rho_t}{\rhotwo_t}\right) d \rho_t \\
& -\operatorname{Var}_{\rho_t}\left(\log\left(\frac{\rho_t}{\pi}\right)\right)+\operatorname{Cov}_{\rho_t}\left(\log\left(\frac{\rho_t}{\pi}\right), \log\left(\frac{\rhotwo_t}{\pitwo}\right)\right)-\operatorname{Cov}_{\rho_t}\left(\log\left(\frac{\rho_t}{\pi}\right), \Vtwo -\Vone\right)\\
&+\left(\E_{\rho_t}\left[\log\left(\frac{\rhotwo_t}{\pitwo}\right)\right]-\E_{\rhotwo_t}\left[\log\left(\frac{\rhotwo_t}{\pitwo}\right)\right]\right).
\end{aligned}
$$

    By Young's inequality $a \cdot b \leq \frac{1}{2c}\|a\|^2+\frac{c}{2}\|b\|^2$, with $a=\nabla \log\left(\frac{\rho_t}{\rhotwo_t}\right)$ and $b=\nabla \Vtwo-\nabla \Vone$ for some $c>0$, and integrating with respect to the density $\rho_t$, we have that
    $$
\int_{\Theta} \left(\nabla \Vtwo-\nabla \Vone\right) \cdot \nabla \log \left(\frac{\rho_t}{\rhotwo_t}\right) d \rho_t \leq \frac{1}{2c} \int_{\Theta} \left\|\nabla \log\left(\frac{\rho_t}{\rhotwo_t}\right)\right\|^2d\rho_t + \frac{c}{2} \int_{\Theta} \left\|\nabla \Vtwo - \nabla \Vone\right\|^{2} d\rho_t.
    $$
    Choosing $c= \frac{1}{2}$ and using Assumption \ref{Ass: Bound on size of perturbation of gradient of potentials}, we have that
    $$
    \begin{aligned}
    \int_{\Theta} \left(\nabla \Vtwo-\nabla \Vone\right) \cdot \nabla \log \left(\frac{\rho_t}{\rhotwo_t}\right) d \rho_t \leq &\int_{\Theta} \left\|\nabla \log\left(\frac{\rho_t}{\rhotwo_t}\right)\right\|^2d\rho_t + \frac{1}{4} \int_{\Theta} \left\|\nabla \Vtwo - \nabla \Vone\right\|^{2} d\rho_t\\
    \leq & \int_{\Theta} \left\|\nabla \log\left(\frac{\rho_t}{\rhotwo_t}\right)\right\|^2d\rho_t + \frac{1}{4} \Vgradsenbound^2.
    \end{aligned}
    $$
    Consequently, we have that
    $$
    \begin{aligned}
        &-\int_{\Theta} \left\|\nabla \log\left(\frac{\rho_t}{\rhotwo_t}\right)\right\|^2 d \rho_t+\int_{\Theta} \left(\nabla \Vtwo-\nabla \Vone\right) \cdot \nabla \log \left(\frac{\rho_t}{\rhotwo_t}\right) d \rho_t \\
        \leq & \left(\frac{1}{2c}-1\right) \int_{\Theta} \left\|\nabla \log\left(\frac{\rho_t}{\rhotwo_t}\right)\right\|^2d\rho_t + \frac{c}{2} \int_{\Theta} \left\|\nabla \Vtwo - \nabla \Vone\right\|^{2} d\rho_t \\
        \leq & \frac{1}{4}\Vgradsenbound^2.
        \end{aligned}
    $$
    
    Define $\tarloglikbound$ to be any constant such that
    $\left|\log \frac{\pione(\theta)}{\pitwo(\theta)}\right| \leq \tarloglikbound$ for all $\theta \in \Theta$. Then, using Lemmas \ref{Variance and covariance balance bound} and \ref{Variance and covariance bounds at time t},
    $$
    -\operatorname{Var}_{\rho_t}\left(\log\left(\frac{\rho_t}{\pi}\right)\right)+\operatorname{Cov}_{\rho_t}\left(\log\left(\frac{\rho_t}{\pi}\right), \log\left(\frac{\rhotwo_t}{\pitwo}\right)\right) \leq \frac{1}{4}\operatorname{Var}_{\rho_t}\left(\log\left(\frac{\rhotwo_t}{\pitwo}\right)\right) \leq \frac{(R_0^{\prime})^2}{16}\exp(-2t).
    $$
    and
    $$
    \left|\operatorname{Cov}_{\rho_t}(s_t,g)\right| \leq \frac{\tarloglikbound R_0 \exp(-t)}{2} \quad \textrm{ and }\quad\E_{\rho_t}[s_t^{\prime}] - \E_{\rho_t^{\prime}}[s_t^{\prime}] \leq 2R_0^{\prime}\exp(-t).
    $$
    The bound on $H(t)$ follows from integrating both sides of the derived inequality for $H^{\prime}(t)$ and the fact that $H(0)=0$ under the common initialization density condition. Under Assumption \ref{Ass: Bound on size of perturbation of potentials}, Lemma \ref{normalization constant and target density sensitivity bound} gives $\tarloglikbound=2\Vsenbound$, while under Assumption \ref{Ass: Alternative Bound on size of perturbation of potentials}, Lemma \ref{target density sensitivity bound under tighter control on oscillation of potentials} gives $\tarloglikbound=\Vsenboundalt$. Applying the two choices of $\tarloglikbound$ above completes the proof.

\end{proof}

\subsubsection{Proof of Theorem \ref{Time derivative of KL divergence between perturbed flows under LSI assumption at every timepoint}} \label{Proof: Time derivative of KL divergence between perturbed flows under LSI assumption at every timepoint}
\begin{proof}
    Under the given conditions, Proposition \ref{Closed form for time derivative of KL divergence between perturbed SHK flows} gives that
    $$
\begin{aligned}
H^{\prime}(t)= &\frac{d}{dt}\KLf{\rho_t}{\rhotwo_t} \\= & -\int_{\Theta} \left\|\nabla \log\left(\frac{\rho_t}{\rhotwo_t}\right)\right\|^2 d \rho_t+\int_{\Theta} \left(\nabla \Vtwo-\nabla \Vone\right) \cdot \nabla \log \left(\frac{\rho_t}{\rhotwo_t}\right) d \rho_t \\
& -\operatorname{Var}_{\rho_t}\left(\log\left(\frac{\rho_t}{\pi}\right)\right)+\operatorname{Cov}_{\rho_t}\left(\log\left(\frac{\rho_t}{\pi}\right), \log\left(\frac{\rhotwo_t}{\pitwo}\right)\right)-\operatorname{Cov}_{\rho_t}\left(\log\left(\frac{\rho_t}{\pi}\right), \Vtwo -\Vone\right)\\
&+\left(\E_{\rho_t}\left[\log\left(\frac{\rhotwo_t}{\pitwo}\right)\right]-\E_{\rhotwo_t}\left[\log\left(\frac{\rhotwo_t}{\pitwo}\right)\right]\right).
\end{aligned}
$$

    By Young's inequality $a \cdot b \leq c\|a\|^2+\frac{1}{4c}\|b\|^2$, with $a=\nabla \log\left(\frac{\rho_t}{\rhotwo_t}\right)$ and $b=\nabla \Vtwo-\nabla \Vone$ for some $0<c<1$, and integrating with respect to the density $\rho_t$, we have that
    $$
\int_{\Theta} \left(\nabla \Vtwo-\nabla \Vone\right) \cdot \nabla \log \left(\frac{\rho_t}{\rhotwo_t}\right) d \rho_t \leq c \int_{\Theta} \left\|\nabla \log\left(\frac{\rho_t}{\rhotwo_t}\right)\right\|^2d\rho_t + \frac{1}{4c} \int_{\Theta} \left\|\nabla \Vtwo - \nabla \Vone\right\|^{2} d\rho_t.
    $$
    Using Assumption \ref{Ass: Bound on size of perturbation of gradient of potentials}, we have that
    $$
    \begin{aligned}
    \int_{\Theta} \left(\nabla \Vtwo-\nabla \Vone\right) \cdot \nabla \log \left(\frac{\rho_t}{\rhotwo_t}\right) d \rho_t \leq &c\int_{\Theta} \left\|\nabla \log\left(\frac{\rho_t}{\rhotwo_t}\right)\right\|^2d\rho_t + \frac{1}{4c} \int_{\Theta} \left\|\nabla \Vtwo - \nabla \Vone\right\|^{2} d\rho_t\\
    \leq & c\int_{\Theta} \left\|\nabla \log\left(\frac{\rho_t}{\rhotwo_t}\right)\right\|^2d\rho_t + \frac{1}{4c} \Vgradsenbound^2.
    \end{aligned}
    $$
    Define $I(t) =\int_{\Theta} \left\|\nabla \log\left(\frac{\rho_t}{\rhotwo_t}\right)\right\|^2 d \rho_t$. Then, we have that
    $$
    \begin{aligned}
        &-\int_{\Theta} \left\|\nabla \log\left(\frac{\rho_t}{\rhotwo_t}\right)\right\|^2 d \rho_t+\int_{\Theta} \left(\nabla \Vtwo-\nabla \Vone\right) \cdot \nabla \log \left(\frac{\rho_t}{\rhotwo_t}\right) d \rho_t\\
        \leq & \left(c-1\right) \int_{\Theta} \left\|\nabla \log\left(\frac{\rho_t}{\rhotwo_t}\right)\right\|^2d\rho_t + \frac{1}{4c} \int_{\Theta} \left\|\nabla \Vtwo - \nabla \Vone\right\|^{2} d\rho_t \\
        \leq & -(1-c)I(t) + \frac{1}{4c}\Vgradsenbound^2.
        \end{aligned}
    $$    
    Under Assumption \ref{Ass: LSI for targets} and Lemma \ref{Fisher information bounded in terms of KL using LSI} and Proposition \ref{LSI of target and bounded log-likelihood of iterate to target leads to LSI of iterate}, we have that $$I(t) \geq  2\exp(-\exp(-t)B) \lambdapi H(t) \geq 2\exp(-B)\lambdapi H(t) = 2 \lambda_{\SHK}^*H(t)$$ where $\lambda_{\SHK}^* = \exp(-B)\lambdapi$ and consequently,
    $$-(1-c) I(t) \leq -2(1-c)\exp(-\exp(-t)B) \lambdapi H(t) \leq -2(1-c)\lambda_{\SHK}^* H(t).
    $$

    Define $\tarloglikbound$ to be any constant such that
    $\left|\log \frac{\pione(\theta)}{\pitwo(\theta)}\right| \leq \tarloglikbound$ for all $\theta \in \Theta$. Then, using Lemmas \ref{Variance and covariance balance bound} and \ref{Variance and covariance bounds at time t}, we have that
    $$
    -\operatorname{Var}_{\rho_t}\left(\log\left(\frac{\rho_t}{\pi}\right)\right)+\operatorname{Cov}_{\rho_t}\left(\log\left(\frac{\rho_t}{\pi}\right), \log\left(\frac{\rhotwo_t}{\pitwo}\right)\right) \leq \frac{1}{4}\operatorname{Var}_{\rho_t}\left(\log\left(\frac{\rhotwo_t}{\pitwo}\right)\right) \leq \frac{(R_0^{\prime})^2}{16}\exp(-2t).
    $$
    and
    $$
    \left|\operatorname{Cov}_{\rho_t}(s_t,g)\right| \leq \frac{\tarloglikbound R_0 \exp(-t)}{2} \quad \textrm{ and }\quad\E_{\rho_t}[s_t^{\prime}] - \E_{\rho_t^{\prime}}[s_t^{\prime}] \leq 2R_0^{\prime}\exp(-t).
    $$

    Applying Lemma \ref{Gronwall's inequality} with $c(t) =  2(1-c)\exp(-\exp(-t)B)\lambdapi$ and $b(t) = b_c(t) = \frac{1}{4c} \Vgradsenbound^2+\frac{(R_0^{\prime})^2}{16}\exp(-2t)+\frac{\tarloglikbound R_0 \exp(-t)}{2}+2R_0^{\prime}\exp(-t)$, along with the fact that $H(0)=0$, we have that
    $$
   H(t) = \KLf{\rho_t}{\rhotwo_t} \leq \int_0^t \exp \left(-2(1-c) \int_s^t \exp(-\exp(-u)B)\lambdapi d u\right) b_{c}(s) d s.
    $$

    For the rest of the proof, note that by applying Lemma \ref{Gronwall's inequality} with constant $c(t) \equiv \kappa$ and $H(0)=0$, we have that
    $$
    H(t) \leq \int_0^t e^{-\kappa(t-s)}\left(A_1+A_2 e^{-s}+A_3 e^{-2 s}\right) d s
    $$
    where $$
    A_1\coloneqq\frac{\Vgradsenbound^2}{4 c}, \quad A_2\coloneqq B\left(\frac{\tarloglikbound}{2} +2\right), \quad A_3\coloneqq\frac{B^2}{16}, \quad \kappa\coloneqq 2(1-c) \exp(-B)\lambdapi .
    $$
    and then the result follows by evaluating each integral explicitly. Under Assumption \ref{Ass: Bound on size of perturbation of potentials}, Lemma \ref{normalization constant and target density sensitivity bound} gives $\tarloglikbound=2\Vsenbound$, while under Assumption \ref{Ass: Alternative Bound on size of perturbation of potentials}, Lemma \ref{target density sensitivity bound under tighter control on oscillation of potentials} gives $\tarloglikbound=\Vsenboundalt$. Applying the two choices of $\tarloglikbound$ above completes the proof.

\end{proof}

\subsubsection{Application to Differential Privacy results}

\paragraph{Proof of Theorem \ref{Pure DP bound under sensitivity bound on potentials}} \label{Proof: Pure DP bound under sensitivity bound on potentials}

\begin{proof}
    Let $\tarloglikbounddata$ be any constant such that
    $$
    \left|\log \frac{\piDone(\theta)}{\piDtwo(\theta)}\right| \leq \tarloglikbounddata \quad \text{for all } \theta \in \Theta .
    $$
    Using Equation \ref{loglikelihood bound pointwise in time}, as derived in the proof of Theorem \ref{log likelihood and Renyi divergence bound only under potential perturbation bound}, we have that
    $$
    \begin{aligned}
    \left|\log \frac{\rho_t(\theta)}{\rhotwo_t(\theta)}\right| \leq \tarloglikbounddata + \exp(-t)(R_{\Done}+R_{\Dtwo}).
    \end{aligned}
    $$
    Using the fact that $R_{\Done},R_{\Dtwo} \leq B$, we obtain that
    $$
    \begin{aligned}
    \left|\log \frac{\rho_t^{\Done}(\theta)}{\rho_t^{\Dtwo}(\theta)}\right| \leq \tarloglikbounddata + \exp(-t) \left(R_{\Done} + R_{\Dtwo}\right) \leq \tarloglikbounddata + 2B\exp(-t)\eqqcolon \varepsilon(t).
    \end{aligned}
    $$
    Since the inequality holds pointwise, we have for all $\theta \in \Theta$,

$$
\rho_t^{\Done}(\theta) \leq \exp(\varepsilon(t))\rho_t^{\Dtwo}(\theta)
$$

with $\varepsilon(t)=\tarloglikbounddata+2B\exp(-t)$. Then, for all measurable $S$, we have that

$$
\mathbb{P}\left(\theta_t \in S \mid \Done\right)=\int_S \rho_t^{\Done}(\theta) d \theta \leq \exp(\varepsilon(t)) \int_S \rho_t^{\Dtwo}(\theta) d \theta= \exp(\varepsilon(t))\mathbb{P}\left(\theta_t \in S \mid \Dtwo\right)
$$
Therefore, the mechanism $M_t(\Done)$ satisfies $(\varepsilon(t), 0)$-DP. Under Assumption \ref{Ass: Bound on size of perturbation of potentials}, Lemma \ref{normalization constant and target density sensitivity bound} gives $\tarloglikbounddata=2\Vsenbound$, while under Assumption \ref{Ass: Alternative Bound on size of perturbation of potentials}, Lemma \ref{target density sensitivity bound under tighter control on oscillation of potentials} gives $\tarloglikbounddata=\Vsenboundalt$. Applying these two choices of $\tarloglikbounddata$ completes the proof. 
\end{proof}

\paragraph{Proof of Theorem \ref{Approximate DP bound under LSI condition}} \label{Proof: Approximate DP bound under LSI condition}

\begin{proof}
    Let $\bar{H}(t)$ denote any upper bound on $H(t) \coloneqq \KLf{\rho_t^{\Done}}{\rho_t^{\Dtwo}}$. Then, using Lemma \ref{Hockey stick divergence and KL divergence comparison}, we have that
    $$
    \mathbb{P}\left(\mathrm{M}_t(D) \in S\right)=P(S) \leq \exp(\varepsilon)Q(S)+\frac{\bar{H}(t)}{\varepsilon}=\exp(\varepsilon) \mathbb{P}\left(\mathrm{M}_t\left(D^{\prime}\right) \in S\right)+\frac{\bar{H}(t)}{\varepsilon}.
    $$
    Therefore, the mechanism $M_t(\Done)$ satisfies $(\varepsilon, \frac{\bar{H}(t)}{\varepsilon})$-DP.

    Now, we consider the upper bound proved in Theorem \ref{Time derivative of KL divergence between perturbed flows under LSI assumption at every timepoint} as $\bar{H}(t)$. This completes the proof.
\end{proof}

\subsubsection{Auxiliary results}

\begin{lemma}\label{SHK flow preseves mass}
    If the initial density $\rho_0 \in \Prob$ i.e. $\int_{\Theta} d\rho_0= 1$ and the sequence of densities $\left\{\rho_t\right\}_{t \geq 0}$ satisfies the transport-reaction equation $\partial_t \rho_t=\nabla \cdot(\nabla \rho_t+\rho_t \nabla V)-\alpha_t \rho_t$, where $\alpha_t(\theta)=\log \rho_t(\theta)-\log \pi(\theta)-\int_{\Theta}(\log \rho_t-\log \pi) d\rho_t $ and $\pi(\theta) \propto \exp(-V(\theta))$ is the Gibbs target determined by the potential $V$, then $\frac{d}{dt}\int_{\Theta} d\rho_t=0$ and hence $\int_{\Theta} d\rho_t= 1$ for all $t \geq 0$.
\end{lemma}
\begin{proof}
    Proof follows from Lemma 3.5 of \citet{kondratyev2019spherical}.
\end{proof}

\begin{lemma}\label{Variance is upper bounded in terms of max and min}
    Let $X$ be any real random variable such that $X \in[m, M]$ almost surely. Then
    \begin{equation*}
    \operatorname{Var}(X) \leq \frac{(M-m)^2}{4}.
    \end{equation*}
\end{lemma}
\begin{proof}
    The proof is extremely elementary but we still provide it for completeness. Let $\mu\coloneqq\E[X] \in[m, M]$. Note that, since $X \in [m,M]$, we have that
    $$
    0 \leq \E[(X-m)(M-X)]=(\mu-m)(M-\mu)-\operatorname{Var}(X),
    $$
    and hence
    $$
    \operatorname{Var}(X) \leq(\mu-m)(M-\mu) \leq \frac{(M-m)^2}{4} .
    $$
\end{proof}

\begin{lemma}\label{Variance and covariance balance bound}
For all $t \geq 0$,
$$
-\operatorname{Var}_\rho(s)+\operatorname{Cov}_\rho\left(s, s^{\prime}\right) \leq \frac{1}{4} \operatorname{Var}_\rho\left(s^{\prime}\right)
$$
\end{lemma}
\begin{proof}
    Let $a\coloneqq\sqrt{\operatorname{Var}_\rho(s)} \geq 0, b\coloneqq\sqrt{\operatorname{Var}_\rho\left(s^{\prime}\right)} \geq 0$. By Cauchy-Schwarz,
    $$
    \operatorname{Cov}_\rho\left(s, s^{\prime}\right) \leq\left|\operatorname{Cov}_\rho\left(s, s^{\prime}\right)\right| \leq \sqrt{\operatorname{Var}_\rho(s)} \sqrt{\operatorname{Var}_\rho\left(s^{\prime}\right)}=a b
    $$
    Therefore
    $$
    -\operatorname{Var}_\rho(s)+\operatorname{Cov}_\rho\left(s, s^{\prime}\right) \leq-a^2+a b
    $$
    Now completing the square, we have that
    $$
    -a^2+a b=-\left(a-\frac{b}{2}\right)^2+\frac{b^2}{4} \leq \frac{b^2}{4}=\frac{1}{4} \operatorname{Var}_\rho\left(s^{\prime}\right).
    $$
\end{proof}

\begin{lemma}\label{identity for likelihood ratio of SHK flow solution with respect to target}
    For $\pi = \frac{\exp(-V)}{Z}$ where $Z = \int_{\Theta} \exp(-V(\theta))d \theta$ and $u= \frac{\rho}{\pi}$, the following identity holds:
    $$
    \nabla \rho+\rho \nabla V=\pi \nabla u .
    $$
\end{lemma}
\begin{proof}
    Using $\rho=\pi u$, we have that
    $$
    \nabla \rho=\nabla(\pi u)=u \nabla \pi+\pi \nabla u .
    $$
    Also,
    $$
    \rho \nabla V=(\pi u) \nabla V .
    $$
    So
    $$
    \nabla \rho+\rho \nabla V=u \nabla \pi+\pi \nabla u+\pi u \nabla V .
    $$
    Since $\pi=Z^{-1} \exp(-V)$ and $Z$ is a constant,
    $$
    \nabla \pi=Z^{-1} \nabla\left(\exp(-V)\right)=Z^{-1}\left(-\exp(-V) \nabla V\right)=-\pi \nabla V.
    $$
    Therefore, we have that
    $$
    \nabla \rho+\rho \nabla V = -\pi u \nabla V+\pi \nabla u+\pi u \nabla V = \pi \nabla u.
    $$
\end{proof}

\begin{lemma}\label{SHK PDE for likelihood and log likelihood}
     Assume $\rho_t \in C^{1,2}([0, T] \times \Theta)$ is a strictly positive solution of the transport-reaction equation \eqref{eq: SHK FPE for potential V} i.e. $ \partial_t \rho_t=\nabla \cdot(\nabla \rho_t+\rho_t \nabla V)-\alpha_t \rho_t$ where the (centered) birth-death rate is $\alpha_t(\theta)=\log \rho_t(\theta)-\log \pi(\theta)-\int_{\Theta}(\log \rho_t-\log \pi) d\rho_t$. Then $u_t=\frac{\rho_t}{\pi}$ satisfies the PDE
    $$
    \partial_t u_t=\pi^{-1} \nabla \cdot\left(\pi \nabla u_t\right)-u_t \log u_t+u_t \E_{\rho_t}\left[\log u_t\right].
    $$
    Further, assume $u_t=\frac{\rho_t}{\pi}>0$ and is smooth. Let $s_t=\log u_t$. Then $s_t$ satisfies the PDE
    $$
    \partial_t s_t=\Delta s_t+\left\|\nabla s_t\right\|^2+\nabla(\log \pi) \cdot \nabla s_t-s_t+\bar{s}_t,
    $$
    where $\bar{s}_t=\int_{\Theta}s_t d\rho_t = \E_{\rho_t}\left[s_t\right]=\E_{\rho_t}\left[\log u_t\right]$.
    
\end{lemma}

\begin{proof}
    The transport-reaction equation is given by 
    $$
    \partial_t \rho_t=\nabla \cdot\left(\nabla \rho_t+\rho_t \nabla V\right)+\rho_t\left(\log \pi-\log \rho_t\right)-\rho_t \E_{\rho_t}\left[\log \pi-\log \rho_t\right].
    $$ The LHS of the above equation is
    $$
    \partial_t \rho_t=\partial_t\left(\pi u_t\right)=\pi \partial_t u_t.
    $$
    Now, we simplify the RHS. Using Lemma \ref{identity for likelihood ratio of SHK flow solution with respect to target}, we have that
    $\nabla \cdot\left(\nabla \rho_t+\rho_t \nabla V\right)=\nabla \cdot\left(\pi \nabla u_t\right)$. Further, note that
    $$
    \begin{aligned}
    \rho_t\left(\log \pi-\log \rho_t\right)=& \left(\pi u_t \right)\left(\log \pi-\left(\log \pi+\log u_t\right)\right)\\
    =&\left(\pi u_t\right)\left(-\log u_t\right)\\
    =&-\pi u_t \log u_t .
    \end{aligned}
    $$
    and $\E_{\rho_t}\left[\log \pi-\log \rho_t\right]=\E_{\rho_t}\left[-\log u_t\right]=-\E_{\rho_t}\left[\log u_t\right]$.Consequently, we have that
    $$
\pi \partial_t u_t=\nabla \cdot\left(\pi \nabla u_t\right)-\pi u_t \log u_t+\pi u_t \E_{\rho_t}\left[\log u_t\right] .
    $$
    and, dividing both sides by $\pi>0$, we have that,
    \begin{equation}\label{PDE for u}
    \partial_t u_t=\pi^{-1} \nabla \cdot\left(\pi \nabla u_t\right)-u_t \log u_t+u_t \E_{\rho_t}\left[\log u_t\right].
    \end{equation}
    Write $u_t=\exp(s_t)$. Then
    $$
    \partial_t u_t=\partial_t\left(\exp(s_t)\right)=\exp(s_t) \partial_t s_t
    $$
    Also $\log u_t=s_t$, so $u \log u_t=\exp(s_t) s_t$. Now the weighted diffusion term can be expressed as
    $$
    \pi^{-1} \nabla \cdot(\pi \nabla u_t)=\pi^{-1} \nabla \cdot\left(\pi \nabla\left(\exp(s_t)\right)\right) .
    $$
    Using $\nabla\left(\exp(s_t)\right)=\exp(s_t) \nabla s_t$, we have that
    $$
    \pi \nabla\left(\exp(s_t)\right)=\pi \exp(s_t) \nabla s_t .
    $$
    Now taking divergence, we have that
    $$
    \nabla \cdot\left(\pi \exp(s_t) \nabla s_t\right)=\left(\nabla\left(\pi \exp(s_t)\right)\right) \cdot \nabla s_t+\pi \exp(s_t) \Delta s_t .
    $$
    Note that $\nabla\left(\pi \exp(s_t)\right)=\exp(s_t) \nabla \pi+\pi \exp(s_t) \nabla s_t$. Therefore, we have that
    $$
    \left(\nabla\left(\pi \exp(s_t)\right)\right) \cdot \nabla s_t=\left(\exp(s_t) \nabla \pi\right) \cdot \nabla s_t+\left(\pi \exp(s_t) \nabla s_t\right) \cdot \nabla s_t=\exp(s_t) \nabla \pi \cdot \nabla s_t+\pi \exp(s_t)\|\nabla s_t\|^2 .
    $$
    So
    $$
    \nabla \cdot\left(\pi \nabla\left(\exp(s_t)\right)\right)=\pi \exp(s_t) \Delta s_t+\pi \exp(s_t)\|\nabla s_t\|^2+\exp(s_t) \nabla \pi \cdot \nabla s_t .
    $$
    Dividing by $\pi$, we have that
    $$
    \pi^{-1} \nabla \cdot\left(\pi \nabla\left(\exp(s_t)\right)\right)=\exp(s_t) \Delta s_t+\exp(s_t)\|\nabla s_t\|^2+\exp(s_t) \frac{\nabla \pi}{\pi} \cdot \nabla s_t=\exp(s_t)\left(\Delta s_t+\|\nabla s_t\|^2+\nabla(\log \pi) \cdot \nabla s_t\right) .
    $$
    Now plugging everything into the PDE given by Equation \eqref{PDE for u}, we have that
    $$
    \exp(s_t) \partial_t s_t=\exp(s_t)\left(\Delta s_t+\|\nabla s_t\|^2+\nabla(\log \pi) \cdot \nabla s_t\right)-\exp(s_t) s_t+\exp(s_t) \bar{s_t}
    $$
    Finally, dividing by $\exp(s_t)>0$, we obtain
    $$
    \partial_t s_t=\Delta s_t+\|\nabla s_t\|^2+\nabla(\log \pi) \cdot \nabla s_t-s_t+\bar{s_t}.
    $$
\end{proof}

\begin{lemma}\label{Bound on Dini derivatives of log likelihood}
    Let $s \in C^{1,2}([0, T] \times \Theta)$.
    Define
    $$
    M(t)\coloneqq\max _{\theta \in \Theta} s(t, \theta), \quad m(t)\coloneqq\min _{\theta \in \Theta} s(t, \theta) .
    $$
    Then for each $t \in[0, T]$, the upper Dini derivative satisfies
    $$
    D^{+} M(t) \leq \max _{\theta \in \operatorname{Argmax} s(t, \cdot)} \partial_t s(t, \theta),
    $$
    and the lower Dini derivative satisfies
    $$
    D^{-} m(t) \geq \min _{\theta \in \operatorname{Argmin} s(t, \cdot)} \partial_t s(t, \theta),
    $$
    where
    $$
    D^{+} M(t)=\underset{h \downarrow 0}{\lim \sup } \frac{M(t+h)-M(t)}{h}, \quad D^{-} m(t)=\underset{h \downarrow 0}{\lim \inf } \frac{m(t+h)-m(t)}{h} .
    $$
\end{lemma}

\begin{proof}
    Because $\Theta$ is compact and $s(t, \cdot)$ is continuous, maxima and minima are attained for each fixed $t$. 

   Fix $t$. For each $h>0$, choose $\theta_h \in \Theta$ such that
    $$
    M(t+h)=s\left(t+h, \theta_h\right) .
    $$
    Then, by definition of $M(t)$ as a maximum at time $t$, we have that
    $$
    M(t)=\max _\theta s(t, \theta) \geq s\left(t, \theta_h\right) .
    $$
    Therefore
    $$
    M(t+h)-M(t)=s\left(t+h, \theta_h\right)-M(t) \leq s\left(t+h, \theta_h\right)-s\left(t, \theta_h\right) .
    $$
    Dividing both sides by $h>0$, we have that
    $$
    \frac{M(t+h)-M(t)}{h} \leq \frac{s\left(t+h, \theta_h\right)-s\left(t, \theta_h\right)}{h} .
    $$
    Now take $\lim \sup _{h \downarrow 0}$ of both sides, we have that
    $$
    D^{+} M(t) \leq \underset{h \downarrow 0}{\lim \sup } \frac{s\left(t+h, \theta_h\right)-s\left(t, \theta_h\right)}{h} .
    $$
    Since $\Theta$ is compact, the sequence $\theta_h$ (as $h \downarrow 0$ ) has accumulation points. Take a sequence $h_n \downarrow 0$ that realizes the limsup and such that $\theta_{h_n} \rightarrow \theta^* \in \Theta$. By continuity of $s$, we have
    $$
    s\left(t, \theta_{h_n}\right) \rightarrow s\left(t, \theta^*\right) .
    $$
    Because $s$ is continuous on the compact set $[0, T] \times \Theta$, it is uniformly continuous. So for every $\varepsilon>0$, there exists $\delta>0$ such that
    $$
    \left|t^{\prime}-t\right|<\delta \Rightarrow \sup _{\theta \in \Theta}\left|s\left(t^{\prime}, \theta\right)-s(t, \theta)\right|<\varepsilon .
    $$
    Therefore,
    $$
    \left|h_n\right| \rightarrow 0, \theta_{h_n} \rightarrow \theta^* \quad \Longrightarrow \quad s\left(t+h_n, \theta_{h_n}\right)-s\left(t, \theta_{h_n}\right) \rightarrow 0 .
    $$
    Let
    $$
    \delta_n\coloneqq s\left(t+h_n, \theta_{h_n}\right)-s\left(t, \theta_{h_n}\right) .
    $$
    Then $\delta_n \rightarrow 0$.
    But by construction $s\left(t+h_n, \theta_{h_n}\right)=M\left(t+h_n\right)$. Hence
    $$
    s\left(t, \theta_{h_n}\right)=s\left(t+h_n, \theta_{h_n}\right)-\delta_n=M\left(t+h_n\right)-\delta_n .
    $$
    Since $s$ is uniformly continuous on $[0, T] \times \Theta$ and $\Theta$ is compact, we have that, for any $t, t^{\prime},\left|M(t)-M\left(t^{\prime}\right)\right| \leq \sup _{\theta \in \Theta}\left|s(t, \theta)-s\left(t^{\prime}, \theta\right)\right|$, and the RHS $\rightarrow 0$ by uniform continuity. Therefore, the function $M(t)=\max _{\theta \in \Theta} s(t, \theta)$ is continuous in $t$.
    Therefore $M\left(t+h_n\right) \rightarrow M(t)$. Combining these results, we have that
    $$
    s\left(t, \theta_{h_n}\right)=M\left(t+h_n\right)-\delta_n \longrightarrow M(t)-0=M(t) .
    $$
    Now passing to the limit using continuity of $s(t, \cdot)$ and $\theta_{h_n} \rightarrow \theta^*$, we have that
    $$
    s\left(t, \theta^*\right)=\lim _{n \rightarrow \infty} s\left(t, \theta_{h_n}\right)=M(t) .
    $$
    Thus $\theta^*$ attains the maximum of $s(t, \cdot)$, i.e.
    $$
    \theta^* \in \operatorname{Argmax} s(t, \cdot) .
    $$
    Now, we proceed to prove that $\lim _{n \rightarrow \infty} \frac{s\left(t+h_n, \theta_{h_n}\right)-s\left(t, \theta_{h_n}\right)}{h_n}=\partial_t s\left(t, \theta^{*}\right) $. Fix $n$. Consider the one-variable function of $\tau \in\left[0, h_n\right]$ :
    $$
    \varphi_n(\tau)\coloneqq s\left(t+\tau, \theta_{h_n}\right)
    $$
    Since $s \in C^{1,2}, \varphi_n$ is $C^1$ in $\tau$ and
    $$
    \varphi_n^{\prime}(\tau)=\partial_t s\left(t+\tau, \theta_{h_n}\right)
    $$
    By the (one-dimensional) Mean Value Theorem, there exists $\xi_n \in\left(0, h_n\right)$ such that
    $$
    \varphi_n\left(h_n\right)-\varphi_n(0)=h_n \varphi_n^{\prime}\left(\xi_n\right) .
    $$
    Expanding and dividing both sides by $h_n>0$, we have that
    $$
    \frac{s\left(t+h_n, \theta_{h_n}\right)-s\left(t, \theta_{h_n}\right)}{h_n}=\partial_t s\left(t+\xi_n, \theta_{h_n}\right) .
    $$
    Now take $n \rightarrow \infty$. We have $\xi_n \in\left(0, h_n\right)$ and $h_n \rightarrow 0$, hence $\xi_n \rightarrow 0$. Also $\theta_{h_n} \rightarrow \theta^*$. Therefore
    $$
    \left(t+\xi_n, \theta_{h_n}\right) \rightarrow\left(t, \theta^*\right)
    $$
    Since $\partial_t s$ is continuous on $[0, T] \times \Theta$, we conclude
    $$
    \partial_t s\left(t+\xi_n, \theta_{h_n}\right) \rightarrow \partial_t s\left(t, \theta^*\right)
    $$
    Thus
    $$
    \lim _{n \rightarrow \infty} \frac{s\left(t+h_n, \theta_{h_n}\right)-s\left(t, \theta_{h_n}\right)}{h_n}=\partial_t s\left(t, \theta^{*}\right)
    $$

    Hence the limsup is bounded by the maximum of $\partial_t s(t, \theta)$ over maximizers $\theta \in \operatorname{Argmax} s(t, \cdot)$. That gives the desired inequality.

    For the lower Dini derivative of the minimum, the argument is analogous: choose $y_h$ achieving the minimum at $t+h$, use $m(t) \leq s\left(t, y_h\right)$, and take lim inf.
\end{proof}

\begin{lemma}\label{Maximum and Minimum conditions on log likelihood}
    Fix $t$. Assume the domain $\Theta$ has no boundary. If $\theta^* \in \operatorname{Argmax} s(t, \cdot)$, then
    $$
    \nabla s\left(t, \theta^*\right)=0, \quad \Delta s\left(t, \theta^*\right) \leq 0 .
    $$
    If $\theta_* \in \operatorname{Argmin} s(t, \cdot)$, then
    $$
    \nabla s\left(t, \theta_*\right)=0, \quad \Delta s\left(t, \theta_*\right) \geq 0
    $$
\end{lemma}

\begin{proof}
    Let $ \phi(\theta)\coloneqq s(t, \theta)$, so $\phi \in C^2(\Theta)$ in the spatial variable (because $s \in C^{1,2}$ ). Since we assume that $\Theta$ is compact and has no boundary, a global maximizer $\theta^*$ (or a global minimizer $\theta_*$) is an interior point of $\Theta$. Therefore, standard interior calculus applies i.e. at an interior maximum of a $C^2$ function, gradient is zero and Hessian is negative semidefinite, implying that $\nabla s(t,\theta^*) = 0$ and $v^\top \nabla^2 s(t,\theta^*) v \leq 0$ for all $v \in \Rd$. Consequently, trace of the Hessian i.e. $\operatorname{Tr}(\nabla^2 s(t,\theta^*)) = \sum_{i=1}^{d} e_i^{\top } \nabla^2 s(t,\theta^*) e_i= \Delta s(t,\theta^*) \leq 0$ where $e_i$ is the $i$-th unit vector in $\R^d$. The case of the Argmin is analogous.
\end{proof}

\begin{lemma}\label{oscillation bound exponential decay}
     Let $s \in C^{1,2}([0, T] \times \Theta)$ where $\Theta$ is a compact subset of $\Rd$. Define
    $$
    M(t)\coloneqq \max _{\theta \in \Theta} s(t, \theta), \quad m(t)\coloneqq \min _{\theta \in \Theta} s(t, \theta) .
    $$
    Then for all $t \in[0, T]$,
$$
\left(M(t) - m(t)\right) \leq \exp(-t) \left(M(0) - m(0)\right) .
$$
\end{lemma}

\begin{proof}
    Let $\theta^* \in \operatorname{Argmax} s(t, \cdot)$. We evaluate the PDE at $\left(t, \theta^*\right)$. By Lemma \ref{Maximum and Minimum conditions on log likelihood}, at $\theta^*$,
    $$
    \nabla s\left(t, \theta^*\right)=0, \quad\left\|\nabla s\left(t, \theta^*\right)\right\|^2=0, \quad \Delta s\left(t, \theta^*\right) \leq 0 .
    $$
    Also $\nabla(\log \pi) \cdot \nabla s_t \lvert_{\theta^*}=0$, since $\nabla s_t \lvert_{\theta^*}=0$. Thus, from Lemma \ref{SHK PDE for likelihood and log likelihood},
$$
\partial_t s\left(t, \theta^*\right)=\Delta s\left(t, \theta^*\right)+0+0-s\left(t, \theta^*\right)+\bar{s}_t \leq-s\left(t, \theta^*\right)+\bar{s}_t .
$$
Since $s\left(t, \theta^*\right)=M(t)$, we get
$$
\partial_t s\left(t, \theta^*\right) \leq-M(t)+\bar{s}_t .
$$
Now applying Lemma \ref{Bound on Dini derivatives of log likelihood}, we have that
$$
D^{+} M(t) \leq \max _{\theta \in \operatorname{Argmax}} \partial_t s(t, \theta) \leq-M(t)+\bar{s}_t .
$$
So we have the inequality
$$
D^{+} M(t) \leq-M(t)+\bar{s}_t .
$$
Let $\theta_* \in \operatorname{Argmin} s(t, \cdot)$. At $\theta_*$, by Lemma \ref{Maximum and Minimum conditions on log likelihood},
$$
\nabla s\left(t, \theta_*\right)=0, \quad\left\|\nabla s\left(t, \theta_*\right)\right\|^2=0, \quad \Delta s\left(t, \theta_*\right) \geq 0
$$
Hence from the PDE,
$$
\partial_t s\left(t, \theta_*\right)=\Delta s\left(t, \theta_*\right)+0+0-s\left(t, \theta_*\right)+\bar{s}_t \geq-s\left(t, \theta_*\right)+\bar{s}_t .
$$
Since $s\left(t, \theta_*\right)=m(t)$,
$$
\partial_t s\left(t, \theta_*\right) \geq-m(t)+\bar{s}_t
$$
Applying Lemma \ref{Bound on Dini derivatives of log likelihood} for the minimum, we have that
$$
D^{-} m(t) \geq \min _{\theta \in \text { Argmin }} \partial_t s(t, \theta) \geq-m(t)+\bar{s}_t
$$
Thus
$$
D^{-} m(t) \geq-m(t)+\bar{s}_t .
$$
Define
$$
y(t)\coloneqq M(t)-m(t).
$$
We want an upper bound on the upper Dini derivative $D^{+} y(t)$.
First, note that using $\lim \sup \left(a_n-b_n\right) \leq \lim \sup a_n-\lim \inf b_n$, we have that
$$
D^{+} y(t)=D^{+}(M(t)-m(t)) \leq D^{+} M(t)-D^{-} m(t) .
$$
Therefore, we have that
$$
D^{+} y(t) \leq\left(-M(t)+\bar{s}_t\right)-\left(-m(t)+\bar{s}_t\right)=-M(t)+m(t)=-(M(t)-m(t))=-y(t) .
$$
So we have the key inequality:
$$
D^{+} y(t) \leq-y(t) .
$$
Define $z(t)\coloneqq \exp(t) y(t)$. Note that,
$$
z(t+h)-z(t)=\exp(t+h) y(t+h)-\exp(t) y(t)=\exp(t)\left(\exp(h) y(t+h)-y(t)\right) .
$$
Divide by $h>0$ :
$$
\frac{z(t+h)-z(t)}{h}=\exp(t)\left(\frac{\exp(h)-1}{h} y(t+h)+\exp(h) \frac{y(t+h)-y(t)}{h}\right) .
$$
Taking $\limsup _{h \downarrow 0}$, using $\lim _{h \rightarrow 0}\left(\exp(h)-1\right) / h=1, \lim _{h \rightarrow 0} \exp(h)=1$, and continuity of $y$ (true since $s$ is continuous and max/min on compact sets vary continuously), we have that
$$
D^{+} z(t)=\underset{h \downarrow 0}{\limsup } \frac{z(t+h)-z(t)}{h}=\exp(t)\left(y(t)+D^{+} y(t)\right) .
$$
Using the inequality $D^{+} y(t) \leq-y(t)$,
$$
D^{+} z(t) \leq \exp(t)(y(t)-y(t))=0 .
$$
Thus $D^{+} z(t) \leq 0$ for all $t$. A function with nonpositive upper Dini derivative is nonincreasing, hence
$$
z(t) \leq z(0) .
$$
That is,
$$
\exp(t) y(t) \leq y(0) \quad \Longrightarrow \quad y(t) \leq \exp(-t) y(0) .
$$
This completes the proof.
\end{proof}

\begin{proposition}\label{Bound on log-likelihood of density at time t to target crude bound}
    Let $R_0\coloneqq \max _{\theta \in \Theta} \log \left(\frac{\rho_0}{\pi}\right) - \min _{\theta \in \Theta} \log \left(\frac{\rho_0}{\pi}\right)$. Then for all $t \geq 0$ and all $\theta \in \Theta$,
    $$
    -\exp(-t)R_0 \leq \log \left( \frac{\rho_t(\theta)}{\pi(\theta)} \right) \leq \exp(-t)R_0.
    $$
\end{proposition}
\begin{proof}
    From Lemma \ref{oscillation bound exponential decay},
$$
\max _{\theta \in \Theta} \log \left(\frac{\rho_t}{\pi}\right)(\theta) - \min _{\theta \in \Theta} \log \left(\frac{\rho_t}{\pi}\right)(\theta) \leq \exp(-t) R_0.
$$
We assume that $\Theta$ is a compact subset of $\Rd$ with no boundary. Now we show $\left|\log \left(\frac{\rho_t}{\pi}\right)(\theta)\right| \leq \max _{\theta \in \Theta} \log \left(\frac{\rho_t}{\pi}\right)(\theta) - \min _{\theta \in \Theta} \log \left(\frac{\rho_t}{\pi}\right)(\theta)$ for all $\theta \in \Theta$. For that, note that, using Lemma \ref{SHK flow preseves mass}, we can show that $u_t=\frac{\rho_t}{\pi}$ satisfies $\int_{\Theta}u_t \pi=1$. Since $u_t>0$, it is impossible that $u_t>1$ everywhere (else the integral would exceed 1), and impossible that $u_t<1$ everywhere (else integral would be $<1$ ). Therefore
$$
\min _\theta u_t(\theta) \leq 1 \leq \max _\theta u_t(\theta) \implies \min _{\theta}\log \left(\frac{\rho_t}{\pi}\right) \leq 0 \leq \max _{\theta}\log \left(\frac{\rho_t}{\pi}\right).
$$
Define $s_t \coloneqq \log \left(\frac{\rho_t}{\pi}\right)$. Then, for all $\theta \in \Theta$,
$$
\left|s_t(\theta)\right| \leq \max \left\{\max s_t,-\min s_t\right\} \leq \max s_t-\min s_t \leq \exp(-t) R_0.
$$
This completes the proof.
\end{proof}

\begin{lemma}\label{normalization constant and target density sensitivity bound}
Define the normalization constants $Z_{\Done} = \int_{\Theta} \exp(-\VDone(\theta))d \theta$ and $Z_{\Dtwo} = \int_{\Theta} \exp(-\VDtwo(\theta))d \theta$ and the Gibbs target distributions $\piDone = \frac{\exp(-\VDone)}{Z_{\Done}}$ and $\piDtwo = \frac{\exp(-\VDtwo)}{Z_{\Dtwo}}$. Let the potentials $\Vone$ and $\Vtwo$ satisfy Assumption \ref{Ass: Bound on size of perturbation of potentials}. Then, we have that $\quad\left|\log \frac{Z_{\Done}}{Z_{\Dtwo}}\right| \leq \Vsenbound$. Further, we have that $\left|\log \frac{\piDone(\theta)}{\piDtwo(\theta)}\right| \leq 2 \Vsenbound$ for any $\theta \in \Theta$.
\end{lemma}

\begin{proof}
    For any $\theta \in \Theta$,
$$
\VDone(\theta) \leq \VDtwo(\theta)+\Vsenbound \quad \Longrightarrow \quad \exp(- \VDone(\theta))\geq \exp(-\Vsenbound) \times \exp(-\VDtwo(\theta)) .
$$
Integrating over $\Theta$, we have that
$$
Z_{\Done}=\int_{\Theta}\exp(- \VDone(\theta)) d\theta \geq \exp(-\Vsenbound) \int_{\Theta}\exp(- \VDtwo(\theta))d\theta=\exp(-\Vsenbound) Z_{\Dtwo}.
$$

Similarly, from $\VDone(\theta) \geq \VDtwo(\theta)-\Vsenbound$, we get $\exp(- \VDone(\theta))\leq \exp(\Vsenbound) \times \exp(-\VDtwo(\theta))$, hence $Z_{\Done}\leq \exp(\Vsenbound) Z_{\Dtwo}$.
Taking logs yields$\quad\left|\log \frac{Z_{\Done}}{Z_{\Dtwo}}\right| \leq \Vsenbound$.

Further, note that,
$$
\log \frac{\piDone(\theta)}{\piDtwo(\theta)}=\log \exp(-\VDone(\theta))-\log Z_{\Done}-\left(\log \exp(-\VDtwo(\theta))-\log Z_{\Dtwo}\right)=-\left(\VDone(\theta)-\VDtwo(\theta)\right)+\log \frac{Z_{\Dtwo}}{Z_{\Done}} .
$$
Taking absolute values and applying triangle inequality, we have that, for any $\theta \in \Theta$,
$$
\left|\log \frac{\piDone(\theta)}{\piDtwo(\theta)}\right| \leq \left|\VDone(\theta)-\VDtwo(\theta)\right| + \left|\log \frac{Z_{\Dtwo}}{Z_{\Done}}\right| \leq \Vsenbound + \Vsenbound = 2\Vsenbound.
$$
\end{proof}

\begin{lemma}\label{target density sensitivity bound under tighter control on oscillation of potentials}
Define the Gibbs target distributions $\piDone = \frac{\exp(-\VDone)}{Z_{\Done}}$ and $\piDtwo = \frac{\exp(-\VDtwo)}{Z_{\Dtwo}}$ and let the potentials $\Vone$ and $\Vtwo$ satisfy Assumption \ref{Ass: Alternative Bound on size of perturbation of potentials}. Then, we have that  $\left|\log \frac{\piDone(\theta)}{\piDtwo(\theta)}\right| \leq  \Vsenboundalt$ for any $\theta \in \Theta$.
\end{lemma}

\begin{proof}
    Note that,
    $$
    \log \frac{\piDone(\theta)}{\piDtwo(\theta)}=\log \exp(-\VDone(\theta))-\log Z_{\Done}-\left(\log \exp(-\VDtwo(\theta))-\log Z_{\Dtwo}\right)=-\left(\VDone(\theta)-\VDtwo(\theta)\right)+\log \frac{Z_{\Dtwo}}{Z_{\Done}} .
    $$    
    Now, let us define $a \coloneqq \inf _{\theta \in \Theta}\left(\VDone(\theta)-\VDtwo(\theta)\right)$ and $b \coloneqq \sup _{\theta \in \Theta}\left(\VDone(\theta)-\VDtwo(\theta)\right)$. Clearly, for any $\theta \in \theta$ , $a \leq \VDone(\theta)-\VDtwo(\theta) \leq b$. Now, we have that
    $$
    \begin{aligned}
    Z_{\Done} = \int_{\Theta} \exp(-\VDone(\theta))d \theta = &Z_{\Dtwo} \int_{\Theta} \exp\left[-\left(\VDone(\theta)-\VDtwo(\theta)\right)\right] \times \frac{\exp(-\VDtwo(\theta))}{Z_{\Dtwo}} d \theta \\
    = & Z_{\Dtwo} \int_{\Theta} \exp\left[-\left(\VDone(\theta)-\VDtwo(\theta)\right)\right] d \pitwo (\theta).
    \end{aligned}
    $$
    Consequently, 
    $$
    \exp(-b) \leq \frac{Z_{\Done}}{Z_{\Dtwo}} \leq \exp(-a) \iff a \leq \log \frac{Z_{\Dtwo}}{Z_{\Done}} \leq b.
    $$
    Therefore, we have that
    $$
    a-b \leq  \log \frac{\piDone(\theta)}{\piDtwo(\theta)} \leq b-a .
    $$
    Finally, under Assumption \ref{Ass: Alternative Bound on size of perturbation of potentials}, we have that, for any $\theta \in \Theta$
    $$
    \left|\log \frac{\piDone(\theta)}{\piDtwo(\theta)}\right| \leq b -a = \sup _{\theta \in \Theta}\left(\Vone(\theta)-\Vtwo(\theta)\right) -  \inf _{\theta \in \Theta}\left(\Vone(\theta)-\Vtwo(\theta)\right) \leq \Vsenboundalt.
    $$
\end{proof}

\begin{proposition}\label{Closed form for time derivative of KL divergence between perturbed SHK flows}
    Assume $\left\{\rho_t\right\}_{t \geq 0}, \left\{\rhotwo_t\right\}_{t \geq 0} \in C^{1,2}([0, T] \times \Theta)$ are strictly positive solutions of the transport-reaction equations \eqref{eq: SHK FPE for potential V} and \eqref{eq: SHK FPE for potential Vprime} on $[0, T]$. Then $H(t)=\mathrm{KL}\left(\rho_t \| \rhotwo_t\right)$ is differentiable and

$$
\begin{aligned}
H^{\prime}(t)= &\frac{d}{dt}\KLf{\rho_t}{\rhotwo_t} \\= & -\int_{\Theta} \left\|\nabla \log\left(\frac{\rho_t}{\rhotwo_t}\right)\right\|^2 d \rho_t+\int_{\Theta} \left(\nabla \Vtwo-\nabla \Vone\right) \cdot \nabla \log \left(\frac{\rho_t}{\rhotwo_t}\right) d \rho_t \\
& -\operatorname{Var}_{\rho_t}\left(\log\left(\frac{\rho_t}{\pi}\right)\right)+\operatorname{Cov}_{\rho_t}\left(\log\left(\frac{\rho_t}{\pi}\right), \log\left(\frac{\rhotwo_t}{\pitwo}\right)\right)-\operatorname{Cov}_{\rho_t}\left(\log\left(\frac{\rho_t}{\pi}\right), \Vtwo -\Vone\right)\\
&+\left(\E_{\rho_t}\left[\log\left(\frac{\rhotwo_t}{\pitwo}\right)\right]-\E_{\rhotwo_t}\left[\log\left(\frac{\rhotwo_t}{\pitwo}\right)\right]\right).
\end{aligned}
$$
\end{proposition}

\begin{proof}
    Define the log-density ratios
$$
w_t\coloneqq\log \frac{\rho_t}{\rhotwo_t}, \quad s_t \coloneqq \log \frac{\rho_t}{\pi}, \quad s_t^{\prime}\coloneqq\log \frac{\rhotwo_t}{\pitwo}, \quad g \coloneqq \log \pi-\log \pitwo .
$$
We have the following identity :
$$
w_t=\log \rho_t-\log \rhotwo_t=(\log \rho_t-\log \pi)-\left(\log \rhotwo_t-\log \pi^{\prime}\right)+\left(\log \pi-\log \pi^{\prime}\right)=s_t-s_t^{\prime}+g .
$$
Further, the centered birth-death rates can be expressed as:
$$
\alpha_t=s_t-\E_{\rho_t}[s_t], \quad \alpha_t^{\prime}=s_t^{\prime}-\E_{\rhotwo_t}\left[s_t^{\prime}\right] .
$$

Now, we have that
$$
H(t)=\int_{\Theta}  \log \frac{\rho_t}{\rhotwo_t} d \rho_t=\int_{\Theta}  w_t d\rho_t
$$
Differentiating, we have that
$$
H^{\prime}(t)=\int_{\Theta}w_t(\theta)\left(\partial_t \rho_t\right(\theta))d\theta +\int_{\Theta}\left(\partial_t w_t\right) d \rho_t
$$
Using $w_t=\log \rho_t-\log \rhotwo_t$, we have that
$$
\partial_t w_t=\frac{\partial_t \rho_t}{\rho_t}-\frac{\partial_t \rhotwo_t}{\rhotwo_t} .
$$
Therefore,
$$
\int_{\Theta}\left(\partial_t w_t\right) d \rho_t=\int_{\Theta}\left(\frac{\partial_t \rho_t}{\rho_t}-\frac{\partial_t \rhotwo_t}{\rhotwo_t}\right) d \rho_t=\int_{\Theta} \partial_t \rho_t(\theta) d \theta-\int_{\Theta} \frac{\partial_t \rhotwo_t}{\rhotwo_t} d \rho_t.
$$
Using Lemma \ref{SHK flow preseves mass}.we have that $\int_{\Theta} \partial_t \rho_t(\theta) d \theta = \frac{d}{dt}\int_{\theta} d\rho_t = 0$. Hence, we have that
$$
H^{\prime}(t)=\int_{\Theta}w_t(\theta)\left(\partial_t \rho_t\right(\theta))d\theta + \int_{\Theta} \partial_t \rho_t(\theta) d \theta-\int_{\Theta} \frac{\partial_t \rhotwo_t}{\rhotwo_t} d \rho_t = \int_{\Theta}w_t(\theta)\left(\partial_t \rho_t\right(\theta))d\theta -\int_{\Theta} \frac{\partial_t \rhotwo_t}{\rhotwo_t} d \rho_t.
$$
We can express the Fokker Planck equations in the following form:
$$
\begin{aligned}
\partial_t \rho_t=A_{\rho_t}-\alpha_t \rho_t, & & A_{\rho_t}\coloneqq \nabla \cdot(\nabla \rho_t+\rho_t \nabla \Vone), \\
\partial_t \rhotwo_t=A_{\rhotwo_t}-\alpha_t^{\prime} \rhotwo_t, & & A_{\rhotwo_t}\coloneqq \nabla \cdot\left(\nabla \rhotwo_t+\rhotwo_t \nabla \Vtwo\right) .
\end{aligned}
$$
Then
$$
\begin{aligned}
H^{\prime}(t)=&\int_{\Theta} A_{\rho_t} w_t d \theta-\int_{\Theta} \frac{A_{\rhotwo_t}}{\rhotwo_t} d \rho_t+\int_{\Theta}(-\alpha_t \rho_t) w_t d \theta-\int_{\Theta} \rho_t \frac{(-\alpha_t^{\prime} \rhotwo_t)}{\rhotwo_t} d \theta\\
=&\underbrace{\int_{\Theta} A_{\rho_t} w_t d \theta-\int_{\Theta} \frac{A_{\rhotwo_t}}{\rhotwo_t} d \rho_t}_{\eqqcolon I_{\mathrm{diff}}}+\underbrace{-\int_{\Theta}\alpha_t w_t d \rho_t+\int_{\Theta} \alpha_t^{\prime}d\rho_t }_{\eqqcolon I_{\mathrm{bd}}}.
\end{aligned}
$$

We now analyze $I_{\mathrm{diff}}$ first using integration by parts and simplifying terms. The first term in $I_{\mathrm{diff}}$ is:
$$
\int_{\Theta} A_{\rho_t} w_t d \theta=\int_{\Theta} \nabla \cdot(\nabla \rho_t+\rho_t \nabla \Vone) w_t d \theta
$$
Using integration by parts gives $\int_{\Theta}(\nabla \cdot F) w_t d \theta=-\int_{\Theta} F \cdot \nabla w_t d\theta$ with $F = \nabla \rho_t+\rho_t \nabla \Vone d$. Thus
$$
\int_{\Theta} A_{\rho_t} w_t d \theta=-\int_{\Theta}(\nabla \rho_t+\rho_t \nabla \Vone) \cdot \nabla w_t d \theta .
$$
The second term in $I_{\mathrm{diff}}$ is:
$$
-\int_{\Theta} \frac{A_{\rhotwo_t}}{\rhotwo_t} d \rho_t=-\int_{\Theta} \frac{\rho_t}{\rhotwo_t} \nabla \cdot\left(\nabla \rhotwo_t+\rhotwo_t \nabla \Vtwo\right) d \theta .
$$
Again, we integrate by parts with $f=\rho_t / \rhotwo_t$ and $G=\nabla \rhotwo_t+\rhotwo_t \nabla \Vtwo$ :
$$
-\int_{\Theta} f \nabla \cdot G d \theta=\int_{\Theta} G \cdot \nabla f d \theta .
$$
So
$$
-\int_{\Theta} \rho_t \frac{A_{\rhotwo_t}}{\rhotwo_t} d \theta=\int_{\Theta}\left(\nabla \rhotwo_t+\rhotwo_t \nabla \Vtwo\right) \cdot \nabla\left(\frac{\rho_t}{\rhotwo_t}\right) d \theta
$$
Now, since $\frac{\rho_t}{\rhotwo_t}=\exp(w_t)$, we have that
$$
\nabla\left(\frac{\rho_t}{\rhotwo_t}\right)=\nabla\left(\exp(w_t)\right)=\exp(w_t)\nabla w_t=\left(\frac{\rho_t}{\rhotwo_t}\right) \nabla w _t.
$$

Therefore

$$
\left(\nabla \rhotwo_t+\rhotwo_t \nabla \Vtwo\right) \cdot \nabla\left(\frac{\rho_t}{\rhotwo_t}\right)=\left(\nabla \rhotwo_t+\rhotwo_t \nabla \Vtwo\right) \cdot\left(\frac{\rho_t}{\rhotwo_t}\right) \nabla w_t=\rho_t\left(\frac{\nabla \rhotwo_t}{\rhotwo_t}+\nabla \Vtwo\right) \cdot \nabla w_t=\rho_t\left(\nabla \log \rhotwo_t+\nabla \Vtwo\right) \cdot \nabla w_t
$$

Also $\nabla \rho_t=\rho \nabla \log \rho_t$, hence

$$
(\nabla \rho_t+\rho_t \nabla \Vone) \cdot \nabla w_t=\rho_t(\nabla \log \rho_t+\nabla \Vone) \cdot \nabla w_t
$$
Combining, we have that
$$
\begin{aligned}
I_{\mathrm{diff}}=&-\int_{\Theta} \rho_t(\nabla \log \rho_t+\nabla \Vone) \cdot \nabla w_t d \theta+\int_{\Theta}\rho_t\left(\nabla \log \rhotwo_t+\nabla \Vtwo\right) \cdot \nabla w_t d \theta\\
=& \int_{\Theta}\rho_t\left[\left(\nabla \log \rhotwo_t+\nabla \Vtwo\right) - (\nabla \log \rho_t+\nabla \Vone)\right] \cdot \nabla w_t d \theta \\
=& \int_{\Theta }\rho_t \left[-\nabla w_t+\left(\nabla \Vtwo-\nabla \Vone\right)\right]\cdot.\nabla w_t d \theta \\
=& - \int_{\Theta }\|\nabla w_t \|^{2}d\rho_t  + \int_{\Theta} \left(\nabla \Vtwo-\nabla \Vone\right)\cdot \nabla w_t d \rho_t .
\end{aligned}
$$

Now, we simplify $I_{\mathrm{bd}}$. Using $w_t=s_t-s_t^{\prime}+g$, $\alpha_t=s_t-\E_{\rho_t}[s_t]$. and $\alpha_t^{\prime}=s_t^{\prime}-\E_{\rho_t^{\prime}}[s_t^{\prime}]$. Let $m_t\coloneqq \E_{\rho_t}[s_t]$ and $m_t^{\prime}\coloneqq \E_{\rho_t^{\prime}}[s_t^{\prime}]$. Then
$$
\begin{aligned}
-\int_{\Theta}\alpha_t w_t d\rho_t=&-\int_{\Theta}(s_t-m_t)\left(s_t-s_t^{\prime}+g\right) d \rho_t\\
=&-\int_{\Theta}(s_t-m_t) s_t d \rho_t+\int_{\Theta}(s_t-m_t) s_t^{\prime} d \rho_t-\int_{\Theta}(s_t-m_t) g d \rho_t \\
=&-\operatorname{Var}_{\rho_t}(s_t)+\operatorname{Cov}_{\rho_t}(s_t,s_t^{\prime}) - \operatorname{Cov}_{\rho_t}(s_t,g).
\end{aligned}
$$
and
$$
\int_{\Theta} \alpha_t^{\prime}d\rho_t = \E_{\rho_t}[s_t^{\prime}] - \E_{\rho_t^{\prime}}[s_t^{\prime}].
$$
Combining, we have that,
$$
I_{\mathrm{bd}} = -\operatorname{Var}_{\rho_t}(s_t)+\operatorname{Cov}_{\rho_t}(s_t,s_t^{\prime}) - \operatorname{Cov}_{\rho_t}(s_t,g) + \left(\E_{\rho_t}[s_t^{\prime}] - \E_{\rho_t^{\prime}}[s_t^{\prime}]\right).
$$
Adding up the expressions for $I_{\textrm{diff}}$ and $I_{\textrm{bd}}$, the proof is complete.
\end{proof}

\begin{lemma}\label{Variance and covariance bounds at time t}
    Define $\tarloglikbound$ to be any constant such that
    $\left|\log \frac{\pione(\theta)}{\pitwo(\theta)}\right| \leq \tarloglikbound$ for all $\theta \in \Theta$. Then, for all $t \geq 0$, we have that $$\operatorname{Var}_{\rho_t}\left(s_t^{\prime}\right) \leq \frac{(R_0^{\prime})^{2}}{4}\exp(-2t) \quad,\quad \operatorname{Cov}_{\rho_t}(s_t,g) \leq \Vsenbound R_0 \exp(-t) \quad \textrm{ and }\quad\E_{\rho_t}[s_t^{\prime}] - \E_{\rho_t^{\prime}}[s_t^{\prime}] \leq 2R_0^{\prime}\exp(-t)$$
    where
    $$
    w_t\coloneqq\log \frac{\rho_t}{\rhotwo_t}, \quad s_t \coloneqq \log \frac{\rho_t}{\pi}, \quad s_t^{\prime}\coloneqq\log \frac{\rhotwo_t}{\pitwo}, \quad g \coloneqq \log \pi-\log \pitwo
    $$
    and
    $$
    R_t\coloneqq \max _{\theta \in \Theta}s_t(\theta) - \min _{\theta \in \Theta}s_t(\theta), \quad R_t^{\prime}\coloneqq \max _{\theta \in \Theta}s_t^{\prime}(\theta) - \min _{\theta \in \Theta}s_t^{\prime}(\theta).
    $$
    Further, the choices $\tarloglikbound = 2 \Vsenbound$ and $\Vsenboundalt$ are valid under the assumptions \ref{Ass: Bound on size of perturbation of potentials} and \ref{Ass: Alternative Bound on size of perturbation of potentials}, respectively.
\end{lemma}

\begin{proof}
    Using Lemma \ref{Variance is upper bounded in terms of max and min} and Proposition \ref{Bound on log-likelihood of density at time t to target crude bound}, we have that $\operatorname{Var}_{\rho_t}\left(s_t^{\prime}\right) \leq \frac{(R_0^{\prime})^{2}}{4}\exp(-2t)$.

    Now, note that by Cauchy-Schwarz,
    $$
    \left|\operatorname{Cov}_{\rho_t}(s_t,g)\right| \leq \sqrt{\operatorname{Var}_{\rho_t}(s_t)} \sqrt{\operatorname{Var}_{\rho_t}(g)} .
    $$
    Using Lemma \ref{Variance is upper bounded in terms of max and min} and Proposition \ref{Bound on log-likelihood of density at time t to target crude bound}, we have $\sqrt{\operatorname{Var}_{\rho_t}(s_t)} \leq \frac{R_0\exp(-t)}{2}$.
    Then, we have that $\sqrt{\operatorname{Var}_{\rho_t}(g)} \leq\|g\|_{\infty} \leq \tarloglikbound$. Thus
    $$
    \left|\operatorname{Cov}_{\rho_t}(s_t, g)\right| \leq \frac{R_0 \exp(-t)\tarloglikbound}{2}.
    $$
    Finally, note that, For any bounded function $h$,
    $$
    \E_\rho[h]-\E_\rhotwo[h] \leq\left|\E_\rho[h]-\E_\rhotwo[h]\right| \leq\left|\E_\rho[h]\right|+\left|\E_\rhotwo[h]\right| \leq 2\|h\|_{\infty} .
    $$
    Applying this observation with $h=s_t^{\prime}$ and using Proposition \ref{Bound on log-likelihood of density at time t to target crude bound} to bound$ \left\|s_t^{\prime}\right\|_{\infty} \leq R_t \exp(-t)$, the result follows with the choice $\tarloglikbound = 2 \Vsenbound$ when Assumption \ref{Ass: Bound on size of perturbation of potentials} holds (based on Lemma \ref{normalization constant and target density sensitivity bound}) while the choice $\tarloglikbound = \Vsenboundalt$ is valid, based on Lemma \ref{target density sensitivity bound under tighter control on oscillation of potentials}, when Assumption \ref{Ass: Bound on size of perturbation of potentials} holds.
\end{proof}

\begin{lemma}\label{Fisher information bounded in terms of KL using LSI}
    Let $\sigma_t$ be the probability measure $d \sigma_t=\rhotwo_t d \theta$ and let $\sigma_t$ satisfy log-Sobolev inequality with constant $\lambda(t)>0$ if for all smooth $f$ with $\int f^2 d \sigma_t=1$,
    $$
    \operatorname{Ent}_{\sigma_t}\left(f^2\right)\coloneqq\int f^2 \log f^2 d \sigma_t \leq \frac{2}{\lambda(t)} \int\|\nabla f\|^2 d \sigma_t
    $$

    Then for any probability density $\rho_t$,
    $$
    I(t)\coloneqq \int \left\|\nabla \log \frac{\rho_t}{\rhotwo_t}\right\|^2 d \rho_t \geq 2 \lambda(t) \mathrm{KL}\left(\rho_t \| \rhotwo_t\right)=2 \lambda(t) H(t)
    $$
\end{lemma}

\begin{proof}
    Let $r\coloneqq \frac{\rho_t}{\rhotwo_t}$. Define $f\coloneqq \sqrt{r}$. Then
$$
f^2=r, \quad \int f^2 d \sigma_t=\int \frac{\rho_t}{\rhotwo_t} \rhotwo_t d x=\int \rho_t d x=1
$$
so $f$ is admissible in the LSI. Then the entropy is given by
$$
\operatorname{Ent}_{\sigma_t}\left(f^2\right)=\int f^2 \log f^2 d \sigma_t=\int r \log r \rhotwo_t d x=\int \rho_t \log \frac{\rho_t}{\rhotwo_t} d x=H(t)
$$
Now, we have that
$$
\nabla \log \frac{\rho_t}{\rhotwo_t}=\nabla \log r=\nabla \log \left(f^2\right)=\frac{2 \nabla f}{f} .
$$
Therefore
$$
\rho_t\left\|\nabla \log \frac{\rho_t}{\rhotwo_t}\right\|^2=\left(f^2 \rhotwo_t\right) \cdot\left\|\frac{2 \nabla f}{f}\right\|^2=f^2 \rhotwo_t \cdot \frac{4\|\nabla f\|^2}{f^2}=4 \rhotwo_t\|\nabla f\|^2 .
$$
Integrating, we have that
$$
I(t)=\int 4 \rhotwo_t|\nabla f|^2 d x=4 \int|\nabla f|^2 d \sigma_t
$$
Now applying LSI, we have that
$$
H(t)=\operatorname{Ent}_{\sigma_t}\left(f^2\right) \leq \frac{2}{\lambda(t)} \int|\nabla f|^2 d \sigma_t=\frac{2}{\lambda(t)} \cdot \frac{I(t)}{4}=\frac{I(t)}{2 \lambda(t)}.
$$
\end{proof}

\begin{lemma}[Gronwall's inequality]\label{Gronwall's inequality}
    Let $H$ be absolutely continuous and satisfy
    $$
    H^{\prime}(t) \leq-c(t) H(t)+b(t)
    $$
    for almost every $t \in[0, T]$, where $c, b$ are integrable on $[0, T]$ and $c(t) \geq 0$. Define
    $$
    C(t)\coloneqq\int_0^t c(u) d u
    $$
    Then for all $t \in[0, T]$,
    $$
    H(t) \leq e^{-C(t)} H(0)+\int_0^t e^{-(C(t)-C(s))} b(s) d s
    $$
    In particular, if $H(0)=0$,
    $$
    H(t) \leq \int_0^t \exp \left(-\int_s^t c(u) d u\right) b(s) d s.
    $$
\end{lemma}

\begin{lemma}[Holley Strook Perturbation principle]\label{Holley Strook Pertubation Lemma}
    \citep[Theorem~A2]{schlichting2019poincare}; \citep{holley1986logarithmic}
    Let $\Omega \subset \mathbb{R}^n$ and $H: \Omega \rightarrow \mathbb{R}$ and $\psi: \Omega \rightarrow \mathbb{R}$ be a bounded function. Let $\mu$ and $\tilde{\mu}$ be the Gibbs measures with potential functions/Hamiltonians $H$ and $H+\psi$, respectively given by
    $$
   \mu \propto \exp( - H) \quad \text { and } \quad \tilde{\mu} \propto \exp(-H-\psi).
    $$
    Then, if $\mu$ satisfies the Log-Sobolev inequality with LSI constant $\alpha$ i.e. for all smooth $f$ with $\int f^2 d \mu=1$,
    $$
    \operatorname{Ent}_{\mu}\left(f^2\right)\coloneqq\int f^2 \log f^2 d \mu \leq \frac{2}{\alpha} \int\|\nabla f\|^2 d \mu
    $$
    then $\tilde{\mu}$ satisfies the Log-Sobolev inequality with LSI constant $\tilde{\alpha}$ which satisfies
    $$
    \tilde{\alpha} \geq \exp(-\operatorname{osc} (\psi)) \alpha
    $$
    where $\operatorname{osc} (\psi)\coloneqq\sup _{\Omega} \psi-\inf _{\Omega} \psi$.
\end{lemma}

\begin{proposition}\label{LSI of target and bounded log-likelihood of iterate to target leads to LSI of iterate}
    Assume $\pione$ and $\pitwo$ satisfy Assumption \ref{Ass: LSI for targets} with the common LSI constant being $\lambdapi>0$. Consider the sequence of SHK gradient flow iterates $\left\{\rho_t\right\}_{t \geq 0}$ and $\left\{\rhotwo_t\right\}_{t \geq 0}$ corresponding to $\pione$ and $\pitwo$. Define
    $$
    s_0(\theta)\coloneqq\log \frac{\rho_0(\theta)}{\pione(\theta)}, \quad s_0^{\prime}(\theta)\coloneqq\log \frac{\rhotwo_0(\theta)}{\pitwo(\theta)}
    $$
    and
    $$
    R_0\coloneqq\max _{\theta \in \Theta}s_0(\theta) - \min _{\theta \in \Theta}s_0(\theta), \quad R_0^{\prime}\coloneqq\max _{\theta \in \Theta}s_0^{\prime}(\theta) - \min _{\theta \in \Theta}s_0^{\prime}(\theta).
    $$
    Further, assume that $R_{0},R_{0} ^{\prime}\leq B$ for some universal constant $B \geq 0$. Then, for every $t\geq 0$, the $t$-th SHK iterate gradient flow iterates $\rho_t$ and $\rhotwo_t$ 
    also satisfy the Log-Sobolev inequality with a common LSI constant $\lambdarho(t)$ that satisfies $\lambdarho(t) \geq \exp(-\exp(-t)B) \lambdapi$. 
    
    Further, uniformly over time $t$, the sequence of SHK gradient flow iterates $\left\{\rho_t\right\}_{t \geq 0}$ and $\left\{\rhotwo_t\right\}_{t \geq 0}$ satisfy the Log-Sobolev inequality with the common LSI constant $\lambdarho$ that satisfies $\lambdarho \geq \lambda^* \coloneqq \exp(-B) \lambdapi$.

    \end{proposition}

\begin{proof}
    Define $s_t(\theta)\coloneqq\log \frac{\rho_t(\theta)}{\pione(\theta)}, \quad s_t^{\prime}(\theta)\coloneqq\log \frac{\rhotwo_t(\theta)}{\pitwo(\theta)}$. We begin the proof by observing that $$\rho_t(\theta) = \exp\left(\log\left(\frac{\rho_t(\theta)}{\pione(\theta)}\right)\right)\pione(\theta) \propto \exp(s_t(\theta))\exp(-\Vone(\theta)) = \exp\left(- \left[\Vone(\theta) + (-s_t(\theta))\right]\right)$$
    and similarly, $$\rhotwo_t(\theta)\propto \exp\left(- \left[\Vtwo(\theta) + (-s_t^{\prime}(\theta))\right]\right)$$

    Note that, using Proposition \ref{Bound on log-likelihood of density at time t to target crude bound}, we have that $$\operatorname{osc} (-s_t) \coloneqq \max _{\theta \in \Theta} (-s_t(\theta)) - \min _{\theta \in \Theta} (-s_t(\theta)) = -\min _{\theta \in \Theta} s_t(\theta) + \max _{\theta \in \Theta} s_t(\theta) \leq \exp(-t) R_{0}\leq \exp(-t)B$$
    and similarly, we have that
    $$\operatorname{osc} (-s_t^{\prime}) \coloneqq \max _{\theta \in \Theta} (-s_t^{\prime}(\theta)) - \min _{\theta \in \Theta} (-s_t^{\prime}(\theta)) = -\min _{\theta \in \Theta} s_t^{\prime}(\theta) + \max _{\theta \in \Theta} s_t^{\prime}(\theta) \leq \exp(-t) R_{0}^{\prime} \leq \exp(-t)B.$$

    Hence, using the Holley-Strook perturbation principle (Lemma \ref{Holley Strook Pertubation Lemma}), we have that, for every $t\geq 0$, the $t$-th SHK iterate gradient flow iterates $\rho_t$ and $\rhotwo_t$ 
    also satisfy the Log-Sobolev inequality with a common LSI constant $\lambdarho(t)$ that satisfies $\lambdarho(t) \geq \exp(-\exp(-t)B) \lambdapi$. 
    
    Further, note that $\lambdarho(t) \geq \exp(-\exp(-t)B) \lambdapi \geq \exp(-B) \lambdapi$ for any $ t \geq 0$. Therefore, uniformly over time $t$, the sequence of SHK gradient flow iterates $\left\{\rho_t\right\}_{t \geq 0}$ and $\left\{\rhotwo_t\right\}_{t \geq 0}$ satisfy the Log-Sobolev inequality with the common LSI constant $\lambdarho$ that satisfies $\lambdarho \geq \exp(-B) \lambdapi$. 
    This completes the proof.
\end{proof}

\begin{lemma}\label{Alternative expression of Hockey stick divergence}
    Assume $P \ll Q$ i.e. $P$ is absolutely continuous with respect to $Q$, where both $P$ and $Q$ are probability measures on $\Theta$. Define the Radon-Nikodym derivative $r\coloneqq \frac{d P}{d Q}$. Then for every $\varepsilon \in \R$,
    $$
    \Hock{P}{Q}=\int_{\Theta}(r(\theta)-\exp(\varepsilon))_{+} d Q(\theta)
    $$
    where $(x)_{+}\coloneqq\max \{x, 0\}$.
\end{lemma}

\begin{proof}
    Fix any measurable $S \subseteq \Theta$. Since $P \ll Q$, we can write
    $$
    P(S)-\exp(\varepsilon)Q(S)=\int_S d P-\exp(\varepsilon)\int_S d Q=\int_S(r(\theta)-\exp(\varepsilon)) d Q(\theta)
    $$
    Pointwise, for every $\theta$, it is easy to verify that
    $$
    (r(\theta)-\exp(\varepsilon)) \mathbf{1}_S(\theta) \leq(r(\theta)-\exp(\varepsilon))_{+},
    $$
    Integrating both sides w.r.t. $Q$ gives
    $$
    P(S)-\exp(\varepsilon)Q(S) \leq \int_{\Theta}(r(\theta)-\exp(\varepsilon))_{+} d Q(\theta)
    $$
    Taking the supremum over $S$ yields
    $$
    \Hock{P}{Q} \leq \int_{\Theta}(r-\exp(\varepsilon))_{+} d Q
    $$
    To show equality, choose the specific set
    $$
    S^*\coloneqq\{\theta: r(\theta)>\exp(\varepsilon)\}
    $$
    On $S^*$ we have $(r-\exp(\varepsilon))_{+}=r-\exp(\varepsilon)$, and on $\left(S^*\right)^c$ we have $(r-\exp(\varepsilon))_{+}=0$. Therefore
    $$
    \int_{S^*}(r-\exp(\varepsilon)) d Q=\int_{\Theta}(r-\exp(\varepsilon))_{+} d Q
    $$
    But the left-hand side is exactly $P\left(S^*\right)-\exp(\varepsilon)Q\left(S^*\right)$. Hence
    $$
    \Hock{P}{Q} \geq P\left(S^*\right)-\exp(\varepsilon)Q\left(S^*\right)=\int_{\Theta}(r-\exp(\varepsilon))_{+} d Q
    $$
    Combined with the earlier inequality, we have that
    $$
    \Hock{P}{Q}=\int_{\Theta}(r-\exp(\varepsilon))_{+} d Q.
    $$
    This completes the proof.
\end{proof}

\begin{lemma}\label{KL as f divergence and properties}
    Let
$$
\psi(x)\coloneqq x\log x-x+1, \quad x \geq 0,
$$
with the convention $0 \log 0\coloneqq0$. Then $\psi(x) \geq 0$ for all $x \geq 0$, and
$$
\KLf{P}{Q}=\int_{\Theta} \psi(r(\theta)) d Q(\theta)
$$
where $r\coloneqq \frac{d P}{d Q}$ is the Radon-Nikodym derivative. Further, consider any $a>1$. Then for all $x \geq 0$,
$$
(x-a)_{+} \leq \frac{\psi(x)}{\log a}
$$
\end{lemma}

\begin{proof} Note that
    $$
    \KLf{P}{Q}=\int_{\Theta} \log \frac{d P}{d Q} d P=\int_{\Theta} \frac{d P}{d Q} \log \frac{d P}{d Q} d Q=\int_{\Theta} \psi(r(\theta)) d Q(\theta)
    $$
    since $\int r d Q=\int d P=1$.

    Now, let us fix $a>1$, so $\log a>0$. Consider $\psi(x)=x \log x-x+1$. For $x>0$,
    $$
    \psi^{\prime}(x)=\log x, \quad \psi^{\prime \prime}(x)=\frac{1}{x}>0,
    $$
    so $\psi$ is convex on $(0, \infty)$ and continuous on $[0, \infty)$ with $\psi(1)=0$ and $\psi(x) \geq 0$.

    Consider the case $x \leq a$. Then $(x-a)_{+}=0$, while $\psi(x) \geq 0$, hence
    $$
    (x-a)_{+}=0 \leq \frac{\psi(x)}{\log a} .
    $$
    Now, consider the case $x \geq a$. By convexity, $\psi$ lies above its tangent at $a$ and hence
    $$
    \psi(x) \geq \psi(a)+\psi^{\prime}(a)(x-a)=\psi(a)+(\log a)(x-a) .
    $$
    Since $\psi(a) \geq 0$, we get
    $$
    \psi(x) \geq(\log a)(x-a) .
    $$
    Rearranging gives
    $$
    x-a \leq \frac{\psi(x)}{\log a}.
    $$
    But when $x \geq a,(x-a)_{+}=x-a$. So
    $$
    (x-a)_{+} \leq \frac{\psi(x)}{\log a} .
    $$
    Combining the two cases proves the inequality for all $x \geq 0$.
\end{proof}

\begin{lemma}\label{Hockey stick divergence and KL divergence comparison}
    Let $P, Q$ be probability measures on the measurable space $(\Theta, \mathcal{B})$ with $P \ll Q$. Fix $\varepsilon>0$. Then, the Hockey-stick divergence and the (exclusive) KL divergence between $P$ and $Q$ satisfies the relation
    $$
    \Hock{P}{Q}\leq \frac{\KLf{P}{Q}}{\varepsilon}.
    $$
    Consequently, for every measurable set $S$, we have that
    $$
    P(S) \leq \exp(\varepsilon) Q(S)+\frac{\KLf{P}{Q}}{\varepsilon}
    $$
\end{lemma}

\begin{proof}
    Let $r=d P / d Q$. By Lemma \ref{Alternative expression of Hockey stick divergence},
$$
\Hock{P}{Q}=\int\left(r-\exp(\varepsilon)\right)_{+} d Q
$$
Applying Lemma \ref{KL as f divergence and properties} with $a=\exp(\varepsilon)$, so $\log a=\varepsilon$, and we obtain the pointwise bound
$$
\left(r-\exp(\varepsilon)\right)_{+} \leq \frac{\psi(r)}{\varepsilon}.
$$
Integrating w.r.t. $Q$, we have that
$$
\Hock{P}{Q}=\int\left(r-\exp(\varepsilon)\right)_{+} d Q \leq \frac{1}{\varepsilon} \int \psi(r) d Q.
$$
But $\int \psi(r) d Q=\KLf{P}{Q}$ because $\int r d Q=1$. Hence
$$
\Hock{P}{Q} \leq \frac{\KLf{P}{Q}}{\varepsilon} .
$$
Finally, since $\Hock{P}{Q}$ is the supremum of $P(S)-\exp(\varepsilon) Q(S)$ over $S$, the bound implies for each measurable $S$,
$$
P(S)-\exp(\varepsilon) Q(S) \leq \frac{\KLf{P}{Q}}{\varepsilon},
$$
i.e.
$$
P(S) \leq \exp(\varepsilon) Q(S)+\frac{\KLf{P}{Q}}{\varepsilon} .
$$
\end{proof}

\begin{lemma}\label{Relation between Total variation and KL}
    Let $P, Q$ be probability measures on the measurable space $(\Theta, \mathcal{B})$. Define total variation
    $$
    \TV{P}{Q}:=\sup _{S \in \mathcal{B}}|P(S)-Q(S)| .
    $$
    If $\TV{P}{Q} \leq \delta$, then for every measurable $S$,
    $$
    P(S) \leq Q(S)+\delta \quad \text { and } \quad Q(S) \leq P(S)+\delta .
    $$
    Equivalently, the (two-sided) $(0, \delta)$-DP inequalities hold for the pair $(P, Q)$.
    Also, 
    $$
        \TV{P}{Q} \leq \sqrt{\frac{\KLf{P}{Q}}{2}}.
    $$
    Finally, let $\varphi: \Theta \rightarrow \R$ be measurable and bounded, $\|\varphi\|_{\infty}<\infty$. Then
    $$
    \left|\E_P[\varphi]-\E_Q[\varphi]\right| \leq 2\|\varphi\|_{\infty} \TV{P}{Q} \leq \|\varphi\|_{\infty} \sqrt{2\KLf{P}{Q}} .
    $$
\end{lemma}
\begin{proof}
    Fix any $S$. By definition of TV,
    $$
    P(S)-Q(S) \leq \sup _A(P(A)-Q(A)) \leq \sup _A|P(A)-Q(A)|=\mathrm{TV}(P, Q) \leq \delta .
    $$
    Rearrange to get $P(S) \leq Q(S)+\delta$.
    Interchanging the roles of $P$ and $Q$ yields $Q(S) \leq P(S)+\delta$.
    Finally, the relation $\TV{P}{Q} \leq \sqrt{\frac{\KLf{P}{Q}}{2}}$ is the well-known Pinsker's inequality.

    Now, using Pinsker's inequality, we have that
    $$\begin{aligned}
    \left|\E_P[\varphi]-\E_Q[\varphi]\right| \leq & \int|\varphi||d(P-Q)| \\
    \leq &\|\varphi\|_{\infty} \int|d(P-Q)| \\
    =& \|\varphi\|_{\infty} \times 2\TV{P}{Q}\\
    \leq & \|\varphi\|_{\infty} \sqrt{2\KLf{P}{Q}}.
    \end{aligned}
    $$
\end{proof}
\subsubsection{Proof of Utility analysis}

Let us define the volume of a near-optimal sublevel set. For $\alpha>0$, define
$$
S_\alpha:=\left\{x \in \Theta: f(x) \leq f_*+\alpha\right\}, \quad m_\alpha:=\operatorname{Vol}\left(S_\alpha\right), \quad m:=\operatorname{Vol}(\Theta) .
$$
where all volumes are with respect to the Lebesgue measure. We first characterize how close is the expected value of the objective function $f$ under the exponential mechanism to the true minimum of the function $f$.

\paragraph{Proof of Proposition \ref{Expected suboptimality of the exponential mechanism}} \label{Proof: Expected suboptimality of the exponential mechanism}

\begin{proof}
Let $X \sim \pi_\beta$. Define $Y:=f(X)-f_* \geq 0$. We first bound the tail probability of $Y$.
Fix $\alpha>0$ and $r \geq 0$. Then

\[
\mathbb{P}(Y \geq \alpha+r)=\pi_\beta\left(\left\{f \geq f_*+\alpha+r\right\}\right)=\frac{1}{Z_\beta} \int_{\left\{f \geq f_*+\alpha+r\right\}} \exp(-\beta f(\theta)) d \theta .
\]

On the set $\left\{f \geq f_*+\alpha+r\right\}$, we have $\exp(-\beta f(\theta)) \leq \exp(-\beta\left(f_*+\alpha+r\right))$. Hence

\begin{equation*}
\int_{\left\{f \geq f_*+\alpha+r\right\}} \exp(-\beta f(\theta)) d \theta \leq \exp(-\beta\left(f_*+\alpha+r\right)) \operatorname{Vol}(\Theta)=\exp(-\beta\left(f_*+\alpha+r\right)) m.
\end{equation*}
Also, we can lower bound $Z_\beta$ using the near-optimal sublevel set $S_\alpha$ as follows: on $S_\alpha, f(x) \leq f_*+\alpha$, so $\exp(-\beta f(\theta))  \geq \exp(-\beta\left(f_*+\alpha\right))$. Therefore
$$
Z_\beta \geq \int_{S_\alpha} \exp(-\beta f(\theta)) d \theta \geq \exp(-\beta\left(f_*+\alpha\right)) \operatorname{Vol}\left(S_\alpha\right)=\exp(-\beta\left(f_*+\alpha\right)) m_\alpha.
$$
Combining, we have that
$$
\mathbb{P}(Y \geq \alpha+r) \leq \frac{\exp(-\beta\left(f_*+\alpha+r\right)) m}{\exp(-\beta\left(f_*+\alpha\right)) m_\alpha}=\frac{m}{m_\alpha} \exp(-\beta r) .
$$
Since probabilities are $ \leq1$, one can sharpen the above bound to,
\begin{equation}\label{eq:Probability bound in utility analysis}
    \mathbb{P}(Y \geq \alpha+r) \leq \min \left(1, \frac{m}{m_\alpha} \exp(-\beta r)\right).
\end{equation}
Computing the expectation using the tail bound, we have that
$
\E[Y]=\int_0^{\infty} \mathbb{P}(Y \geq t) d t.
$
Splitting the integral at $\alpha$,
$$
\E[Y]=\int_0^\alpha \mathbb{P}(Y \geq t) d t+\int_\alpha^{\infty} \mathbb{P}(Y \geq t) d t \leq \alpha+\int_0^{\infty} \mathbb{P}(Y \geq \alpha+r) d r
$$
Using Equation \eqref{eq:Probability bound in utility analysis} with $R\coloneqq m / m_\alpha \geq 1$, we have that
$$
\E[Y] \leq \alpha+\int_0^{\infty} \min \left(1, R \exp(-\beta r)\right) d r.
$$

Let $r_0:=\frac{1}{\beta} \log R$. Then for $0 \leq r \leq r_0, R \exp(-\beta r) \geq 1$, and for $r \geq r_0, R \exp(-\beta r) \leq 1$. Therefore,
$$
\int_0^{\infty} \min \left(1, R \exp(-\beta r)\right) d r=\int_0^{r_0} 1 d r+\int_{r_0}^{\infty} R\exp(-\beta r) d r=r_0+\frac{R}{\beta} \exp(-\beta r_0).
$$
But $\exp(-\beta r_0)=\exp(- \log R)=1 / R$, so the second term is $1 / \beta$. Thus,
$$
\int_0^{\infty} \min \left(1, R \exp(-\beta r)\right) d r=\frac{1}{\beta} \log R+\frac{1}{\beta}
$$
Hence
$$
\E_{\pi_{\beta}}[f]-f_*=\E[Y] \leq \alpha+\frac{1}{\beta}\left(\log \frac{m}{m_\alpha}+1\right).
$$
This completes the proof.
\end{proof}

\paragraph{Proof of Theorem \ref{Utility bound for SHK gradient flow sampler}} \label{Proof: Utility bound for SHK gradient flow sampler}
\begin{proof}
Applying Lemma \ref{Relation between Total variation and KL} with $\varphi=f$ and $\pi=\pi_\beta$ :

$$
\E_{\rho_t}[f] \leq \E_{\pi_\beta}[f]+\left|\E_{\rho_t}[f]-\E_{\pi_\beta}[f]\right| \leq \E_{\pi_\beta}[f]+2\|f\|_{\infty} \sqrt{\frac{1}{2} \operatorname{KL}\left(\rho_t \| \pi_\beta\right)} .
$$
and then subtracting $f_*$ from both sides, we obtain
$$
\E_{\rho_t}[f]-f_* \leq\left(\E_{\pi_\beta}[f]-f_*\right)+2\|f\|_{\infty} \sqrt{\frac{1}{2} \operatorname{KL}\left(\rho_t \| \pi_\beta\right)} .
$$
Using Proposition \ref{Expected suboptimality of the exponential mechanism}, we have that
$$
\E_{\pi_\beta}[f]-f_* \leq \alpha+\frac{1}{\beta}\left(\log \frac{m}{m_\alpha}+1\right).
$$

Finally, by \citet[Theorem 3.3]{lu2019accelerating} applied to $\pi=\pi_\beta$, for $t \geq t_*$,
$$
\mathrm{KL}\left(\rho_t \| \pi_\beta\right) \leq \exp(-(2-3 \delta)\left(t-t_*\right)) \mathrm{KL}\left(\rho_{t_0} \| \pi_\beta\right)= \exp(-(2-3 \delta)\left(t-t_*\right)) H\left(t_0\right) .
$$
\end{proof}

\subsection{Experiments}\label{sec: Experiments}
In this section, we discus some experiments that we performed to validate our theoretical perturbation bound based results and discuss different aspects of our theoretical results.

\paragraph{Setup:} The experiments numerically validate the SHK perturbation theory on the one-dimensional torus $\Theta= [-\pi, \pi)$, a compact boundaryless domain matching the assumptions of the main theorem. All experiments are run on the one-dimensional torus $\Theta=[-\pi, \pi)$ using a periodic midpoint grid. For a grid with $N$ cells, the breakpoints are chosen as $x_i=-\pi+\left(i+\frac{1}{2}\right) \Delta x, \quad \Delta x=\frac{2 \pi}{N}$. The Gibbs targets are computed as $\pi_i=\frac{\exp \left(-V_i\right)}{\sum_j \exp \left(-V_j\right) \Delta x}, \quad \pi_i^{\prime}=\frac{\exp \left(-V_i^{\prime}\right)}{\sum_j \exp \left(-V_j^{\prime}\right) \Delta x}$. The SHK PDE solved numerically is $\partial_t \rho_t=\nabla \cdot\left(\nabla \rho_t+\rho_t \nabla V\right)-\alpha_t \rho_t$, where $\alpha_t(\theta)=\log \frac{\rho_t(\theta)}{\pi(\theta)}-\mathbb{E}_{\rho_t}\left[\log \frac{\rho_t}{\pi}\right]$. The solver uses a periodic finite-volume discretization and RK4 time integration for the SHK transport-reaction PDE. The finite-volume part discretizes the Fokker-Planck transport term, and the reaction term is evaluated pointwise using the centered log-density ratio. The PDEs are not discretized and implemented using discrete particles, but rather they are implemented as direct numerical PDE approximations to the continuum SHK flow so as to keep the correspondence to our continuous-time results as much as possible.

\paragraph{Quantities tracked in experiments: } Let $h(\theta)=V(\theta)-V^{\prime}(\theta)$. The experiments compare the empirical pointwise log-ratio $\widehat{L}_{\mathrm{emp}}(t)=\left\|\log \frac{\rho_t}{\rho_t^{\prime}}\right\|_{\infty}=\max _\theta\left|\log \rho_t(\theta)-\log \rho_t^{\prime}(\theta)\right|
$ against the envelopes and bounds presented in our paper. The initialization oscillations are $R_0=\operatorname{osc}\left(\log \frac{\rho_0}{\pi}\right), \quad R_0^{\prime}=\operatorname{osc}\left(\log \frac{\rho_0}{\pi^{\prime}}\right)$, where $\operatorname{osc}(f)=\sup _\theta f(\theta)-\inf _\theta f(\theta)$. The \ref{Ass: Bound on size of perturbation of potentials} sup-norm sensitivity is $\Delta_{\operatorname{pot}}=\left\|V-V^{\prime}\right\|_{\infty}=\|h\|_{\infty}$. The \ref{Ass: Alternative Bound on size of perturbation of potentials} oscillation sensitivity is $\Delta_{\mathrm{osc}}=\operatorname{osc}\left(V-V^{\prime}\right)=\sup _\theta h(\theta)-\inf _\theta h(\theta)$. Our Theorem \ref{log likelihood and Renyi divergence bound only under potential perturbation bound} gives, under Assumption \ref{Ass: Bound on size of perturbation of potentials}, 
$\left\|\log \frac{\rho_t}{\rho_t^{\prime}}\right\|_{\infty} \leq 2 \Delta_{\sup }+e^{-t}\left(R_0+R_0^{\prime}\right)$. and under Assumption \ref{Ass: Alternative Bound on size of perturbation of potentials}, $\left\|\log \frac{\rho_t}{\rho_t^{\prime}}\right\|_{\infty} \leq \Delta_{\mathrm{osc}}+e^{-t}\left(R_0+R_0^{\prime}\right) $. Theorem \ref{log likelihood and Renyi divergence bound only under potential perturbation bound} also implies the R\'enyi-divergence bound
$D_\alpha\left(\rho_t \| \rho_t^{\prime}\right) \leq L(t)$ for every $\alpha>1$, where $L(t)$ is the corresponding pointwise log-ratio envelope. For comparison purpose, since we are working in an oracle  one-dimensional numerical testbed the exact Gibbs targets are computable, the true target mismatch can be captured by the sharper "exact target-floor" envelope $L_{\text {target }}(t)=\left\|\log \frac{\pi}{\pi^{\prime}}\right\|_{\infty}+e^{-t}\left(R_0+R_0^{\prime}\right)$. For privacy experiments, Theorem \ref{Pure DP bound under sensitivity bound on potentials} gives the pure-DP privacy bounds $\varepsilon_{\mathrm{A} 1}(t)=2 \Delta_{\mathrm{pot}}+2 B e^{-t}$ and $\varepsilon_{\mathrm{A} 1^{\prime}}(t)=\Delta_{\mathrm{pot}, \mathrm{osc}}+2 B e^{-t},
$ where $B$ uniformly bounds the two initialization oscillations $R_D, R_{D^{\prime}}$.

\paragraph{Experiment 1:} For the plot in the left panel of Figure \ref{fig: Experiment 1}, the base potential is taken as $
V(x)=1.2(1-\cos (2 x))$ and the perturbed potential is chosen as $V^{\prime}(x)=V(x)+0.6 \sin x+1$. The \ref{Ass: Bound on size of perturbation of potentials} bound is safe but very conservative here because the additive constant artificially increases $\| V- V^{\prime} \|_{\infty}$. The \ref{Ass: Alternative Bound on size of perturbation of potentials} bound is substantially better because oscillation sensitivity ignores additive constants. The exact target-floor bound is closest to the observed behavior.

\begin{figure}[!htbp]
	\centering
\includegraphics[width=\linewidth]{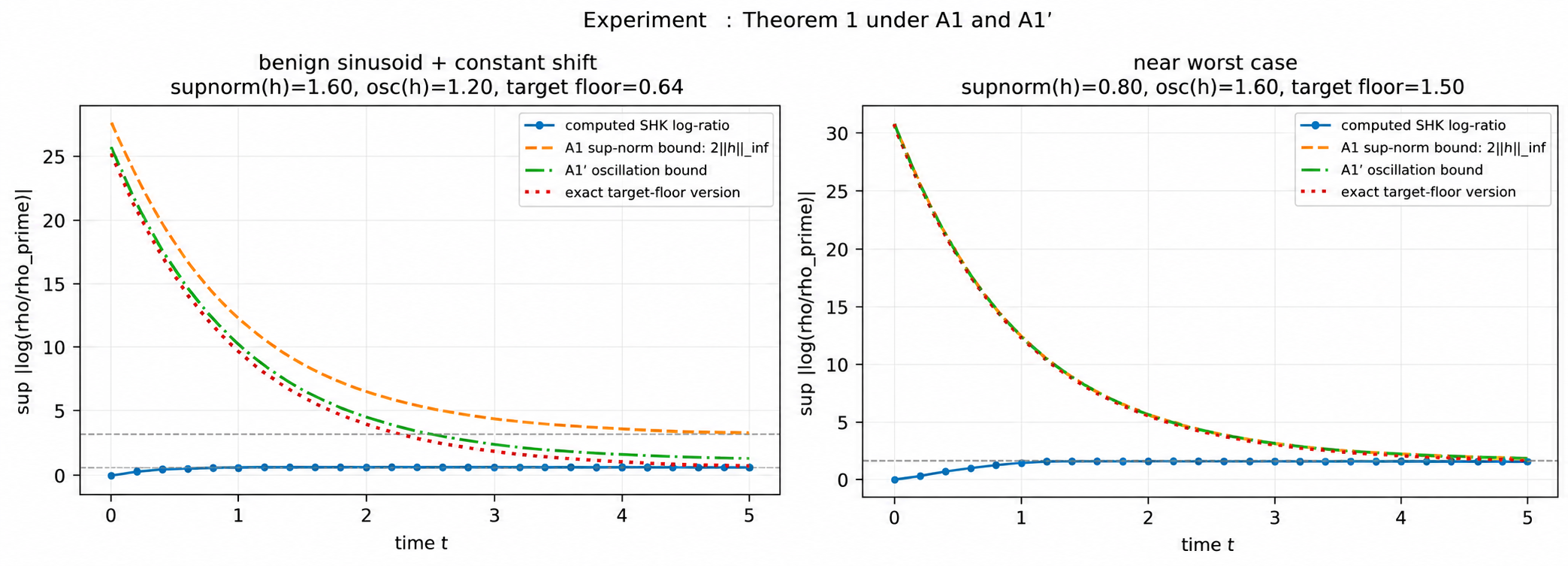}
	\caption{Experiment 1}
	\label{fig: Experiment 1}
\end{figure}

For the plot in the right panel of Figure \ref{fig: Experiment 1}, the base potential is $V(x)=4(1-\cos x)$ and the perturbed potential is $V^{\prime}(x)=V(x)-0.8 \cos x$. Thus, $
h(x)=V(x)-V^{\prime}(x)=0.8 \cos x$. In this case, $2\|h\|_{\infty} \approx \operatorname{osc}(h)$, so the \ref{Ass: Bound on size of perturbation of potentials} and \ref{Ass: Alternative Bound on size of perturbation of potentials} floors essentially coincide. This panel shows that the factor of 2 in the Theorem \ref{Pure DP bound under sensitivity bound on potentials} under Assumption \ref{Ass: Bound on size of perturbation of potentials} is not merely an artifact of the proof. Under absolute sup-norm control, the stationary target mismatch can get close to $2\left\|V-V^{\prime}\right\|_{\infty}$. Therefore the correct conclusion is that one cannot remove a factor of 2 in the Assumption \ref{Ass: Bound on size of perturbation of potentials} based result in Theorem \ref{log likelihood and Renyi divergence bound only under potential perturbation bound}/Theorem \ref{Pure DP bound under sensitivity bound on potentials} i.e. it has a genuine worst-case multiplicative factor of 2, while Assumption \ref{Ass: Alternative Bound on size of perturbation of potentials} allows an intrinsic sharper formulation when oscillation sensitivity is available.

\begin{table}[h]
\centering
\begin{tabular}{lc}
\hline
\textbf{Quantity} & \textbf{Value} \\
\hline
Empirical SHK log-ratio & 0.644016 \\
A1 bound & 3.365895 \\
A1$'$ bound & 1.365895 \\
Exact target-floor bound & 0.809605 \\
$\|V - V'\|_{\infty}$ & 1.599920 \\
$2\|V - V'\|_{\infty}$ & 3.199839 \\
$\operatorname{osc}(V - V')$ & 1.199839 \\
Exact target floor $\|\log(\pi/\pi')\|_{\infty}$ & 0.643549 \\
\hline
\end{tabular}
\caption{Comparison of empirical and theoretical bounds.}
\label{tab:shk-bounds}
\end{table}

\begin{table}[h]
\centering
\begin{tabular}{lc}
\hline
\textbf{Quantity} & \textbf{Value} \\
\hline
Empirical SHK log-ratio & 1.505272 \\
A1 bound & 1.797173 \\
A1$'$ bound & 1.797173 \\
Exact target-floor bound & 1.698790 \\
$\|V - V'\|_{\infty}$ & 0.799893 \\
$2\|V - V'\|_{\infty}$ & 1.599786 \\
$\operatorname{osc}(V - V')$ & 1.599786 \\
Exact target floor & 1.501403 \\
\hline
\end{tabular}
\caption{Comparison of empirical and theoretical bounds.}
\label{tab:shk-bounds-second}
\end{table}

\paragraph{Experiment 2:} The base potential is $V(x)=1.0(1-\cos (2 x))$. The perturbation is $V^{\prime}(x)=V(x)+0.45 \sin x$. The experimental results are reported in Figure \ref{fig: Experiment 2}. In the left panel of Figure \ref{fig: Experiment 2}, the experiment computes R\'enyi divergences $D_\alpha(\rho_t || \rho_t)$ for $\alpha \in\{2,3,5,10\}$ and compares it to the bound $L_{\text {target }}(t)$. The R\'enyi bounds are valid but conservative, especially for small $\alpha$. The conservativeness is expected because R\'enyi divergence is an integral quantity, while Theorem \ref{log likelihood and Renyi divergence bound only under potential perturbation bound} controls a stronger pointwise likelihood-ratio quantity. The bound is tighter for larger $\alpha$, since high-order R\'enyi divergence is more sensitive to pointwise likelihood-ratio peaks.

\begin{table}[h]
\centering
\begin{tabular}{ccc}
\hline
$\alpha$ & Empirical $D_{\alpha}$ & Theorem bound $L_{\mathrm{target}}(t)$ \\
\hline
2  & 0.055881 & 0.635315 \\
3  & 0.083185 & 0.635315 \\
5  & 0.133828 & 0.635315 \\
10 & 0.228820 & 0.635315 \\
\hline
\end{tabular}
\caption{Empirical Rényi divergences and theorem bounds for different values of $\alpha$.}
\label{tab:renyi-theorem-bound}
\end{table}

The KL bound given in Theorem \ref{Time derivative of KL divergence between perturbed flows under LSI assumption at every timepoint} uses a Young-inequality parameter $c \in(0,1)$. In the LSI-based bound, the main quantities are $A_1=\frac{\Delta_{\text {gradpot }}^2}{4 c}$ and $\kappa=2(1-c) e^{-B} \lambda_{\text {Gibbs }}$ and the asymptotic KL plateau behaves like $\frac{A_1}{\kappa}$. The paper explains that the KL bound is controlled by a balance between forcing from gradient perturbations and LSI contractivity. The constant $c$ trades off these two effects: increasing $c$ decreases $A_1$, but also decreases $\kappa$. 

\begin{figure}[!htbp]
	\centering
\includegraphics[width=\linewidth]{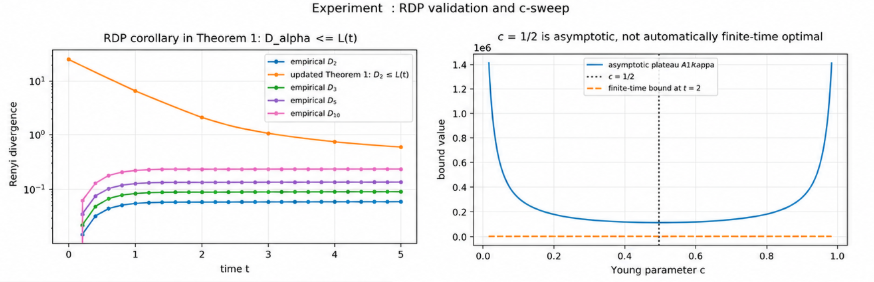}
	\caption{Experiment 2}
	\label{fig: Experiment 2}
\end{figure}

In the right panel plot of Figure \ref{fig: Experiment 2}, the solid curve is the asymptotic plateau $\frac{A_1}{\kappa}$. This is the limiting KL-bound plateau. The vertical dotted line is $c = \frac{1}{2}$. The dashed curve evaluates the explicit finite-time KL bound at $t=2$. At the grid point closest to $c=1 / 2$, the asymptotic plateau is essentially minimized, but the finite-time $t=2$ value is not minimized. The experiment shows that $c=1 / 2$ is the correct asymptotic balancing choice, but a finite-time certificate may benefit from optimizing $c$ numerically for the runtime of interest. This is a useful practical note for using the KL theorem in certificates. Also, the large value of the KL bound at finite time is due to a conservative placeholder for $\lambda_{\text {Gibbs }}$ on the torus. Therefore this panel should be presented as an illustration of the $c$-tradeoff structure, not as a sharp applied KL certificate.

\begin{table}[h]
\centering
\begin{tabular}{lcc}
\hline
\textbf{Quantity optimized} & \textbf{Minimizing $c$} & \textbf{Minimum value} \\
\hline
Asymptotic plateau $A_1/\kappa$ & 0.497588 & 111060.018822 \\
Finite-time bound at $t=2$ & 0.980000 & 29.683066 \\
\hline
\end{tabular}
\caption{Optimized values of $c$ for the asymptotic plateau and finite-time bound.}
\label{tab:optimized-c-values}
\end{table}
\paragraph{Experiment 3: } This is the main applied DP experiment. It studies SHK sampling for an exponential-mechanism target in a synthetic torus mean-estimation problem. The paper's DP application considers dataset-dependent Gibbs targets $\pi_D(\theta) \propto \exp \left(-V_D(\theta)\right)$. Theorem \ref{Pure DP bound under sensitivity bound on potentials} turns the pointwise SHK stability bound into pure-DP guarantees, and Theorem 5 separates exponential-mechanism intrinsic utility from finite-time sampling error. The dataset has $n=100$ torus-valued observations near 0.25 . A neighboring dataset $D^{\prime}$ is created by changing one observation to -2.4. The loss is $L_D(\theta)=\frac{1}{n} \sum_{i=1}^n\left(1-\cos \left(\theta-y_i\right)\right)$. The exponential-mechanism potential is $V_D(\theta)=\beta L_D(\theta)$ for $\beta=5$. The neighboring potential is $V_{D^{\prime}}(\theta)=\beta L_{D^{\prime}}(\theta)$. The experimental results are reported in Figure \ref{fig: Experiment 3}.

\begin{figure}[!htbp]
	\centering
\includegraphics[width=0.8\linewidth]{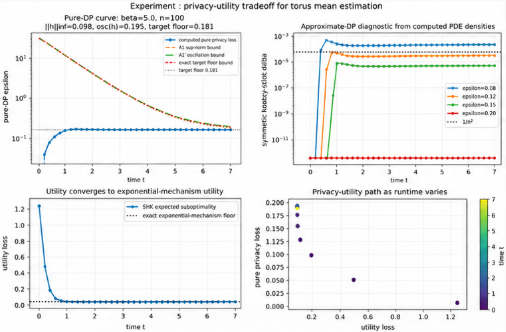}
	\caption{Experiment 3}
	\label{fig: Experiment 3}
\end{figure}

At $t=7$, the computed pure privacy loss is 0.1816 , while the \ref{Ass: Bound on size of perturbation of potentials}/\ref{Ass: Alternative Bound on size of perturbation of potentials} theorem bound is 0.2201 and the exact target-floor envelope is 0.2058 as shown in the top-left panel. The final utility loss is 0.10746 , essentially matching the exact exponential-mechanism utility floor 0.10755 . The computed symmetric hockey-stick divergence at $\varepsilon=$ 0.15 is $6.57 \times 10^{-6}$, below $1 / n^2=10^{-4}$ which shows the attainability of $\delta$ values smaller than $\frac{1}{n^2}$ as enquired by the reviewer. Choosing the utility diagnostic to be $U(t)=\mathbb{E}_{\rho_t^D}\left[L_D(\theta)\right]-\min _\theta L_D(\theta)$ with the exact exponential-mechanism utility floor being $
U_\pi=\mathbb{E}_{\pi_D}\left[L_D(\theta)\right]-\min _\theta L_D(\theta) .
$, this experiment also demonstrates that the SHK runtime traces a clear privacy-utility tradeoff-path: early stopping of the SHK dynamics gives smaller privacy loss but worse utility, while longer runtime approaches the exponential-mechanism utility floor and target privacy floor.

\paragraph{Experiment 4: } This experiment compares SHK/WFR against the ordinary Langevin/Fokker-Planck flow on a metastable bimodal target. The goal is to show that SHK's reaction term can give much faster convergence to the target, which improves the utility side of the privacy-utility story. The base potential is $V(x)=2.5(1-\cos (2 x))$. The perturbation is $V^{\prime}(x)=V(x)+0.35 \sin x$. The experiment evolves four PDEs: SHK flow under $V$, SHK flow under $V^{\prime}$, Langevin/Fokker-Planck flow under $V$ and Langevin/Fokker-Planck flow under $V^{\prime}$. The experimental results are reported in Figure \ref{fig: Experiment 4}.

\begin{figure}[!htbp]
	\centering
\includegraphics[width=0.8\linewidth]{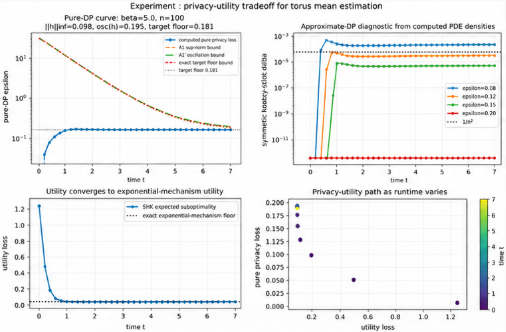}
	\caption{Experiment 4}
	\label{fig: Experiment 4}
\end{figure}

At $t=6$, the SHK pointwise log-ratio is 0.362 , below the SHK theorem envelope 0.414, demonstrating that the SHK theorem bound remains valid and reasonably tight. The stability of SHK and Langevin is similar in this particular perturbation diagnostic, but SHK converges to the target much faster and achieves a much lower value of the KL divergence from the Gibbs target as well, since SHK reaches $K L\left(\rho_t \| \pi\right)=2.91 \times 10^{-5}$, whereas the Langevin method remains stuck at 0.371. This supports the utility/convergence advantage of the SHK reaction term, even in the context of metastable bimodal distribution landscape.


\end{document}